\documentclass{article}

\usepackage[nonatbib,preprint]{neurips_2024}
\usepackage[numbers]{natbib}

\usepackage[utf8]{inputenc} % allow utf-8 input
\usepackage[T1]{fontenc}    % use 8-bit T1 fonts
\usepackage{hyperref}       % hyperlinks
\usepackage{url}            % simple URL typesetting
\usepackage{booktabs}       % professional-quality tables
\usepackage{amsfonts}       % blackboard math symbols
\usepackage{nicefrac}       % compact symbols for 1/2, etc.
\usepackage{microtype}      % microtypography
\usepackage{xcolor}         % colors
\usepackage{bm}
\usepackage{mathtools}
\usepackage{amsmath}
\usepackage{amssymb}
\usepackage{amsthm}
\usepackage[toc,page,header]{appendix}
\usepackage{minitoc}

\DeclareMathOperator*{\argmin}{arg\,min}

\newtheorem{theorem}{Theorem}
\newtheorem{proposition}{Proposition}
\newtheorem{corollary}{Corollary}
\newtheorem{definition}{Definition}

\title{Density Ratio Estimation via Sampling along Generalized Geodesics on Statistical Manifolds}

\author{Masanari Kimura, Howard Bondell \\
      \texttt{\{m.kimura, howard.bondell\}@unimelb.edu.au} \\
      School of Mathematics and Statistics \\
      The University of Melbourne
      }

\begin{document}

\doparttoc % Tell to minitoc to generate a toc for the parts
\faketableofcontents % Run a fake tableofcontents command for the partocs

\part{} % Start the document part
% \parttoc

\maketitle

\begin{abstract}
The density ratio of two probability distributions is one of the fundamental tools in mathematical and computational statistics and machine learning, and it has a variety of known applications. Therefore, density ratio estimation from finite samples is a very important task, but it is known to be unstable when the distributions are distant from each other. One approach to address this problem is density ratio estimation using incremental mixtures of the two distributions. We geometrically reinterpret existing methods for density ratio estimation based on incremental mixtures. We show that these methods can be regarded as iterating on the Riemannian manifold along a particular curve between the two probability distributions. Making use of the geometry of the manifold, we propose to consider incremental density ratio estimation along generalized geodesics on this manifold. To achieve such a method requires Monte Carlo sampling along geodesics via transformations of the two distributions. We show how to implement an iterative algorithm to sample along these geodesics and show how changing the distances along the geodesic affect the variance and accuracy of the estimation of the density ratio. Our experiments demonstrate that the proposed approach outperforms the existing approaches using incremental mixtures that do not take the geometry of the manifold into account.
\end{abstract}

\section{Introduction}
\label{sec:introduction}

The density ratio of two probability distributions is one of the fundamental tools in mathematical and computational statistics.
For example, its applications include estimation under the covariate shift assumption~\citep{shimodaira2000improving,sugiyama2007direct}, outlier detection~\citep{azmandian2012local,hido2011statistical}, variable selection~\citep{azmandian2012gpu,heck2019caveat,oh2016bayesian}, change point detection~\citep{hushchyn2021generalization,kawahara2009change,liu2013change} and causal inference~\citep{matsushita2023estimating,reichenheim2010measures}.
For practical purposes, since the true density ratio is not accessible, density ratio estimation from finite size samples is required.
Therefore, density ratio estimation (DRE) is a very important task that has attracted significant interest~\citep{kanamori2010theoretical,sugiyama2012density,yamada2013relative}.

A key challenge in DRE is the difficulty of the estimation when the source and target distributions are far apart.
One approach to address this problem is DRE using incremental mixtures of two distributions.
\citet{rhodes2020telescoping} have shown the effectiveness of a method that creates $T > 0$ bridging distributions connecting the source and target distributions and then shrinks the gap between the two distributions by utilizing a divide-and-conquer framework.
The method is based on the simple fact that a chain of density ratios using arbitrary bridge distributions restores the original density ratio.
They employ a linear mixture of source and target distributions as the bridge distributions.
In addition, \citet{choi2022density} extended this method to use infinitely continuum bridge distributions.
We refer to the estimator that these methods produce as the Incremental Mixture Density Ratio Estimator (IMDRE).

We investigate the family of IMDRE within the framework of information geometry~\citep{amari2016information,amari2000methods}, which allows for discussion on manifolds created by a set of probability distributions.
From the viewpoint of information geometry, it is known that the set of probability distributions constitutes a Riemannian manifold, which is a non-Euclidean space, and that the parameter space of the probability distributions plays the role of a coordinate system.
With this geometric tool, we can reinterpret IMDRE with linear mixtures as a sequential DRE along a specific curve called the $m$-geodesic.
A geodesic here is defined as a straight path on a manifold when a certain coordinate system is used.
This means that different coordinate systems are equipped with different geodesics.
We propose IMDRE with generalized geodesics called $\alpha$-geodesics.
We call this extension as Generalized IMDRE (GIMDRE).
To realize GIMDRE, sampling along the $\alpha$-geodesics connecting the source and target distributions is required.
However, since the classes of the two probability distributions and their parameters are unknown, direct sampling is a non-trivial problem.
To address this issue, by utilizing the Monte Carlo method with importance sampling~\citep{elvira2014advances,tokdar2010importance}, we show that sampling according to $\alpha$-geodesics can be achieved by giving the density ratio of the source and target distributions.
The next problem here is that the density ratio estimation by GIMDRE relies on sampling along the $\alpha$-geodesic, while the sampling depends on the density ratio.
To resolve this interdependent deadlock, we develop an optimization algorithm that alternately estimates the density ratio and updates the importance weights for sampling.
We design numerical experiments to demonstrate the usefulness of the proposed method and its behavior.
To summarize, our contributions are as follows.
\begin{itemize}
    \item We geometrically reinterpret IMDRE through the lens of information geometry, and we show that IMDRE with linear mixtures can be viewed as a sequential DRE along a special curve called the $m$-geodesic. Furthermore, asymptotic analysis shows that certain geodesics appear in the evaluation of the DRE (Section~\ref{sec:geometric_reinterpretation_imdre}, Theorems~\ref{thm:linear_imdre}, ~\ref{thm:asymptotic_expansion} and Corollary~\ref{thm:asymptotic_hellinger}).
    \item We consider extending IMDRE with generalized geodesics called $\alpha$-geodesics. We refer to this extension as GIMDRE, and we show that the algorithm to obtain GIMDRE can be achieved by alternating optimization of density ratio estimation and updating the importance weighting for sampling based on the Monte Carlo procedure (Section~\ref{sec:algorithm_gimdre} and Figure~\ref{fig:alternating_algorithm}). % We also show that the variance and effective sample size depend on what geodesic is used.
    \item We design numerical experiments, including permutation test, an important task in statistics, to investigate the behavior of our algorithm and demonstrate its effectiveness (Section~\ref{sec:numerical_experiments}, Figures~\ref{fig:toy_example_gaussian}, ~\ref{fig:alpha_vs_kld}, ~\ref{fig:ptest}, and Tables~\ref{tab:gaussian_n_steps}, ~\ref{tab:sample_size_alpha}, ~\ref{tab:dimension_alpha}).
\end{itemize}
Our study provides insights into the connections between statistical procedures and differential geometry.
We show that the use of $\alpha$-geodesics for the estimation of density ratios improves upon both the traditional IMDRE and the direct estimation.
See the appendix for technical details such as proofs of theorems and experimental setups, as well as additional results including discussion on effective sample size, analysis of variance, and additional numerical experiments.

\section{Background and Preliminary}
\label{sec:background}
In this section, we first provide the background and preliminary knowledge that is required throughout this study.
See Appendix~\ref{apd:related_work} for related literature.
\subsection{Problem Setting}
\label{sec:problem_setting}
Let $\mathcal{X}\subset\mathbb{R}^d$ be the $d$-dimensional data domain with $d \in \mathbb{N}$.
Suppose that we are given independent and identically distributed (i.i.d.) samples $\{\bm{x}_i^s\}^{n_s}_{i=1}$ and $\{\bm{x}_i^t\}^{n_t}_{i=1}$ of sizes $n_s, n_t \in \mathbb{N}$, from $p_s(\bm{x})$ and $p_t(\bm{x})$, the source and target distributions, respectively.
The goal is to estimate the density ratio
\begin{align}
    r(\bm{x}) \coloneqq p_s(\bm{x}) / p_t(\bm{x}), \label{eq:density_ratio}
\end{align}
where we assume that $p_t(\bm{x})$ is strictly positive over the domain $\mathcal{X}$.
Let $\hat{r}(\bm{x})$ be the estimator of $r(\bm{x})$, and what we want to achieve is to get $\hat{r}(\bm{x})$ and $r(\bm{x})$ as close as possible.

\subsection{Incremental Mixture Density Ratio Estimation}
\label{sec:imdre}
The IMDRE family includes Telescoping DRE~\citep{rhodes2020telescoping} and $\infty$-DRE~\citep{choi2022density}, and its formulation is based on a density ratio chain with arbitrary bridge distributions as follows.
\begin{align}
    r(\bm{x}) = \frac{p_s(\bm{x})}{p_t(\bm{x})} = \frac{p_s(\bm{x})}{p_1(\bm{x})}\frac{p_1(\bm{x})}{p_2(\bm{x})}\cdots\frac{p_{m-1}(\bm{x})}{p_m(\bm{x})}\frac{p_m(\bm{x})}{p_t(\bm{x})},
\end{align}
where $p_1(\bm{x}),\dots,p_m(\bm{x})$ are $m$ bridge distributions.
They consider the linear mixtures to gradually transport sample $\{\bm{x}_i^s\}^{n_s}_{i=1}$ from the source distribution $p_s(\bm{x})$ to sample $\{\bm{x}_i^t\}^{n_t}_{i=1}$ from the target distribution $p_t(\bm{x})$ as $\bm{x}_i^k = (1 - \lambda)\bm{x}_i^s + \lambda\bm{x}_i^t$ with $\lambda \in (0, 1)$.
Note that \citet{rhodes2020telescoping} use $\lambda = \sqrt{1 - a^2_k}$ where $a_k$ form an increasing sequence from $0$ to $1$.
For each successive density ratio, once the samples are taken, standard approaches can be used to estimate the density ratio, such as kernel logistic regression.
Their experimental results report that such a framework yields good density ratio estimates.

\section{Geometric Reinterpretation of IMDRE}
\label{sec:geometric_reinterpretation_imdre}
% まずはじめに，IMDREの幾何学的な振る舞いを調べるために確率分布の集合が作る統計的多様体を考える．
To begin with, we consider the statistical manifolds constructed by a set of probability distributions in order to investigate the geometric behavior of IMDRE.
Let $\mathcal{M} = \{p(\cdot; \bm{\theta}) \mid \bm{\theta} \in \Theta \}$ be a statistical model parametrized by $\bm{\theta} \in \Theta$, where $\Theta \subset \mathbb{R}^b$ is a parameter space.
Let $\bm{G} = (g_{ij})$ be the Fisher information matrix defined as $g_{ij} = \mathbb{E}[\partial_i\ell(\bm{\theta})\partial_j\ell(\bm{\theta})]$, where $\ell(\bm{\theta}) = \ln p(\bm{x}; \bm{\theta})$ and $\partial_i = \partial / \partial\theta_i$.
Here, it is known that $\mathcal{M}$ is a Riemannian manifold with $\Theta$ as the coordinate system and $\bm{G}$ as coefficients of the metric.
A coordinate system is intuitively like location information for determining a point on a manifold, which can be understood from the fact that determining the parameters determines the corresponding probability distribution.
The metric enables us to measure distances and angles on manifolds and to determine connections, and its requirements are to be symmetric, positive definite, and non-degenerate.
Under appropriate regularity conditions, the Fisher information matrix satisfies these requirements.
% Since the notations on differential geometry are quite massive, in the following we try to use only the minimum necessary for our analysis.
We avoid using geometry notations more than necessary and require only knowledge of basic covariant tensors.
However, additional knowledge of differential geometry is included in Appendix~\ref{apd:summary_of_geometry} to make it self-contained.
For more details, refer to textbooks on differential geometry~\citep{guggenheimer2012differential,kreyszig2013differential} and Riemannian geometry~\citep{gallot1990riemannian,klingenberg1995riemannian,lee2018introduction,sakai1996riemannian}.
Also see Appendix~\ref{apd:proofs} for proofs of the theoretical results in this section.

% alpha-geodesicを定義
Consider a smooth curve $\bm{\theta}(\lambda)\colon [0, 1] \to \Theta$ in the parameter space $\Theta$ with $\lambda \in [0, 1]$, and a curve $\gamma(\lambda)\colon [0, 1] \to \mathcal{M}$ on the statistical manifold $\mathcal{M}$ is defined as
\begin{align}
    \gamma(\lambda) \coloneqq p(\bm{x}; \bm{\theta}(\lambda)), \quad \forall \lambda \in [0, 1]. \label{eq:geodesics}
\end{align}
The velocity of the curve $\gamma(\lambda)$ is given by $\dot{\gamma}(\lambda) = (d / d\lambda)p(\bm{x}; \bm{\theta}(\lambda))$.
A curve $\gamma$ is called $\nabla^{(\alpha)}$-autoparallel if $\dot{\gamma}(\lambda)$ is parallel transported along $\gamma(\lambda)$, that is, the acceleration with respect to the $\nabla^{(\alpha)}$-connection vanishes
\begin{align}
    \nabla^{(\alpha)}_{\dot{\gamma}(\lambda)}\dot{\gamma}(\lambda) = 0, \quad \forall \lambda \in [0, 1].
\end{align}
Here, let $\mathcal{X}(\mathcal{M})$ be a set of vector fields on $\mathcal{M}$, and $\nabla^{(\alpha)}$-connection is defined as
\begin{align}
    \nabla^{(\alpha)}_X \coloneqq \frac{1+\alpha}{2}\nabla^{(1)}_X + \frac{1-\alpha}{2}\nabla^{(-1)}_X, \quad \forall X \in \mathcal{X}(\mathcal{M}),
\end{align}
where $\nabla^{(1)}$ and $\nabla^{(-1)}$ are operators satisfying
\begin{align}
    g(\nabla^{(1)}_{\partial_i}\partial_j, \partial_k) &= \mathbb{E}\left[(\partial_i\partial_j\ell)(\partial_k\ell)\right], \\
    g(\nabla^{(-1)}_{\partial_i}\partial_j, \partial_k) &= \mathbb{E}\left[(\partial_i\partial_j\ell + \partial_i\ell\partial_j\ell)(\partial_k\ell)\right],
\end{align}
and $2$-covariant tensor $g(X, Y)$ is a Riemannian metric on $\mathcal{M}$ as
\begin{align}
    g(X, Y) = \sum^b_{i,j=1}g_{ij}v_iw_j, \quad X = \sum^b_{i=1}v_i\partial_i,\ Y = \sum^n_{i=1}w_i\partial_i.
\end{align}
The operation $\nabla^{(\cdot)}\colon\mathcal{X}(\mathcal{M})\times\mathcal{X}(\mathcal{M})\to\mathcal{X}(\mathcal{M})$ is called a linear connection, and see Appendix~\ref{apd:summary_of_geometry} for the details.
If the curve in Eq.~\eqref{eq:geodesics} is $\nabla^{(\alpha)}$-autoparallel at some $\alpha \in \mathbb{R}$, it is called $\alpha$-geodesic.
The explicit form of $\alpha$-geodesics connecting two probability distributions is given as follows.
\begin{definition}[$\alpha$-geodesics~\citep{amari2016information}]
    Let $p, q \in \mathcal{M}$ be two probability distributions.
    The $\alpha$-geodesic $\gamma^{(\alpha)}(\lambda)\colon [0, 1] \to \mathcal{M}$ connecting $p(\bm{x})$ and $q(\bm{x})$ is defined as
    \begin{align}
        \gamma^{(\alpha)}(\lambda) = \begin{cases}
            \left\{(1-\lambda)p(\bm{x})^{\frac{1-\alpha}{2}} + \lambda q(\bm{x})^{\frac{1-\alpha}{2}}\right\}^{\frac{2}{1-\alpha}}, & (\alpha \neq 1), \\
            \exp\left\{(1 - \lambda) \ln p(\bm{x}) + \lambda \ln q(\bm{x})\right\}, & (\alpha = 1).
        \end{cases}
    \end{align}
    ignoring the normalization factor.
\end{definition}
Given an $\alpha\in\mathbb{R}$, the $\alpha$-geodesics are known to be the minimizer of the following $\alpha$-divergence $D_\alpha [p \| q]$, that is, for a fixed $\lambda \in [0, 1]$, $\gamma^{(\alpha)}(\lambda) = \argmin_{\pi \in \mathcal{M}}((1-\lambda)D_{\alpha}[p\|\pi] + \lambda D_{\alpha}[q\|\pi])$.
% alpha-divergenceがalpha-geodesicのminimizerであるという
\begin{definition}[$\alpha$-divergence~\citep{amari2009alpha}]
    Let $\alpha \in \mathbb{R}$.
    For two probability distributions $p$ and $q$, the $\alpha$-divergence $D_\alpha\colon\mathcal{M}\times\mathcal{M}\to [0,+\infty)$ is defined as
    \begin{align}
        D_\alpha[p \| q] = \frac{1}{\alpha(\alpha - 1)}\left( 1- \int p(\bm{x})^\alpha q(\bm{x})^{1-\alpha}d\bm{x}\right). \label{eq:alpha_geodesics}
\end{align}
\end{definition}

% Theorem: 線型結合のIMDREがe-geodesicに沿った密度比推定
Here, consider the following IMDRE with linear mixtures as bridge distributions.
\begin{align}
    r(\bm{x}) = \frac{p_s(\bm{x})}{p_t(\bm{x})} = \frac{p_s(\bm{x})}{(1-\lambda_1)p_s(\bm{x}) + \lambda_1 p_t(\bm{x})}\cdots\frac{(1-\lambda_m) p_s(\bm{x}) + \lambda_m p_t(\bm{x})}{p_t(\bm{x})}, \label{eq:imdre_linear}
\end{align}
where $\lambda_1,\dots,\lambda_m \in (0, 1)$ is a non-decreasing sequence.
For such an IMDRE, a simple calculation yields the following.
\begin{theorem}
    \label{thm:linear_imdre}
    The IMDRE given by Eq.~\eqref{eq:imdre_linear} can be regarded as a DRE along a geodesic with $\alpha = -1$.
    Additionally, the curve to which the bridge distributions of this IMDRE belong is the minimizer of the KL-divergence between $p_s$ and $p_t$.
\end{theorem}
Recall that $\pi^*$ is a minimizer of KL-divergence $D_{\mathrm{KL}}\colon\mathcal{M}\times\mathcal{M}\to[0, +\infty)$ with respect to $\lambda$, $p_s$ and $p_t$ if it satisfies $\pi^* = \argmin_{\pi \in \mathcal{M}} (1 - \lambda)D_{\mathrm{KL}}[p_s \| \pi] + \lambda D_{\mathrm{KL}}[p_t \| \pi]$.
Here, the geodesics of $\alpha=-1$ are specially called $m$-geodesics ($m$ stands for mixture).
Similarly, the geodesics of $\alpha=+1$ are specially called $e$-geodesics ($e$ stands for exponential).

% Theorem: 任意の密度比推定量の精度は漸近的に2-geodesicに沿って評価される
Now consider the evaluation of the density ratios obtained.
Given that the Eq.~\eqref{eq:density_ratio} can be transformed as in $p_s(\bm{x}) = r(\bm{x})p_t(\bm{x})$, a well-estimated density ratio should minimize
\begin{align}
    D_{\mathrm{KL}}[p_s \| \hat{r}\cdot p_t] \coloneqq \int_{\mathcal{X}}p_s(\bm{x})\log\frac{p_s(\bm{x})}{\hat{r}(\bm{x})p_t(\bm{x})}dx - 1 + \int_{\mathcal{X}} p_t(\bm{x})\hat{r}(\bm{x})d\bm{x}, \label{eq:d_ukl}
\end{align}
where $D_{\mathrm{KL}}[p \| q]$ is the unnormalized version of KL-divergence between $p$ and $q$.
We can then show the following theorem.
\begin{theorem}
    \label{thm:asymptotic_expansion}
    The evaluation along the m-geodesics of any density ratio estimator $\hat{r}(\bm{x})$ is asymptotically evaluated along the $\alpha$-geodesics with $\alpha=2$.
\end{theorem}
% Theorem: 0-geodesicに沿って評価される任意の密度比推定量はO(n^{-1 / 2 + zeta} + \sqrt{\zeta})
Furthermore, if the parameter $\alpha$ that determines the geodesic is well determined, the divergence that appears can be bounded above by the Lipschitz continuity.
In fact, KL-divergence is not Lipschitz continuous since $x \to 0$ and $\ln x \to +\infty$, but the $\alpha$-divergence in the case of $\alpha=0$ satisfies this.
Utilizing this, we obtain the following results from \citet{sugiyama2008direct}.
\begin{corollary}
    \label{thm:asymptotic_hellinger}
    Under the appropriate assumptions, the convergence bound of a density ratio estimator $\hat{r}$ evaluated along $\alpha$-geodesic with $\alpha=0$ is $O_p(n^{-1 / (2 + \zeta_n)} + \sqrt{\zeta})$, where $\zeta_n$ is a variable depending on the sample size $n$.
    %\begin{align}
    %    \zeta_n \coloneqq \max\left\{\frac{1}{n_s}\sum^{n_s}_{i=1}\ln\hat{r}(\bm{x}^s_i) + \frac{1}{n_s}\sum^{n_t}_{i=1}\ln \frac{r(\bm{x}^s_i)}{\frac{1}{n_t}\sum^{n_t}_{i=1}r(\bm{x}^t_i)}, 0\right\},
    %\end{align}
    and $\zeta$ is some constant.
\end{corollary}
These results suggest the following:
\begin{itemize}
    \item Choosing how to construct the bridge distributions in IMDRE is equivalent to choosing what curve on the statistical manifold to follow for density ratio estimation.
    \item The Behavior of the density ratio estimator depends on which curve to construct the bridge distributions along. In fact, Theorem~\ref{thm:asymptotic_expansion} shows that in asymptotic situations geodesics with $\alpha=2$ and $\alpha=0$ appear.
\end{itemize}
These suggestions naturally motivate the development of IMDRE along general $\alpha$-geodesics.
% The problem is that sampling from the $\alpha$-geodesics given by Eq.~\eqref{eq:alpha_geodesics} is non-trivial.
% If the true distributions $p_s$ and $p_t$ are known, we can construct the $\alpha$-geodesics explicitly, but the problem we want to address now is to estimate the density ratio when $p_s$ and $p_t$ are unknown.

\section{Optimization Algorithm for GIMDRE along Generalized Geodesics}
\label{sec:algorithm_gimdre}
\begin{figure}
    \centering
    \includegraphics[width=\linewidth]{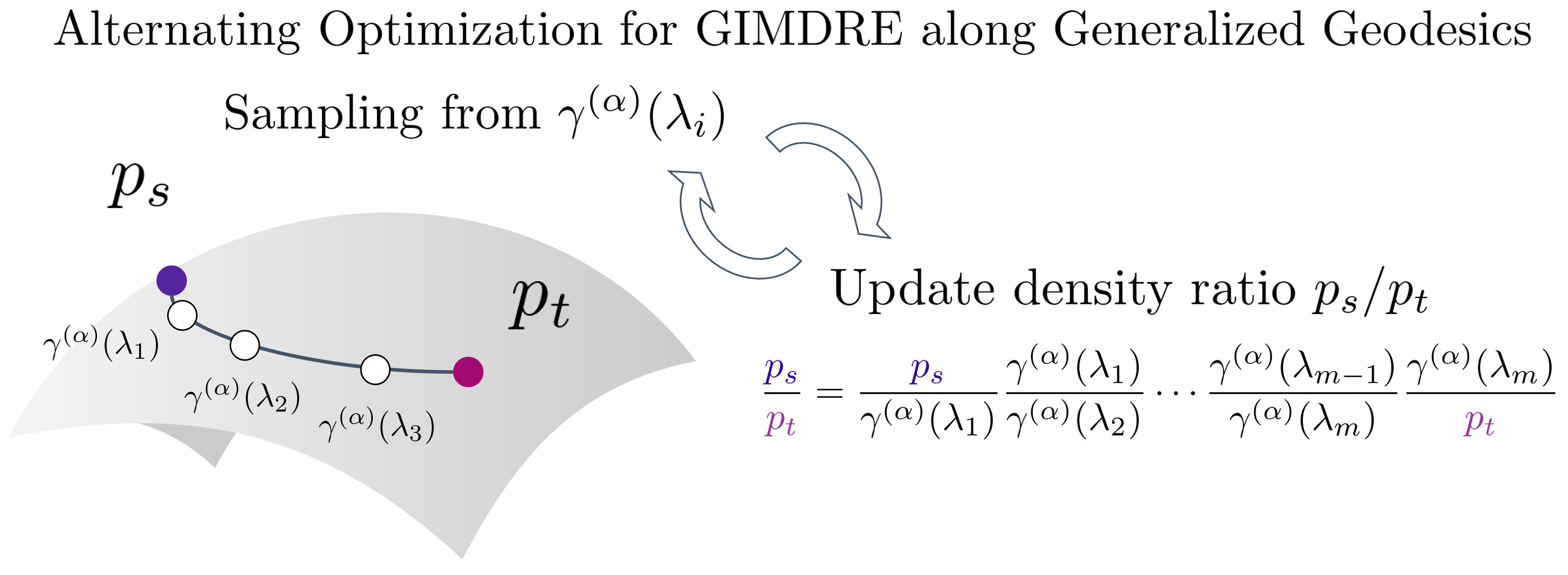}
    \caption{Overview of alternating optimization algorithms for GIMDRE along $\alpha$-geodesics.
    $\gamma^{(\alpha)}(\lambda_1),\dots,\gamma^{(\alpha)}(\lambda_m)$ are $m$ bridge distributions on the $\alpha$-geodesic parameterized by $\lambda \in [0, 1]$, respectively.}
    \label{fig:alternating_algorithm}
\end{figure}
% alpha-geodesicsに沿ったDREを考える．
We now construct IMDRE along general $\alpha$-geodesics, Generalized IMDRE (GIMDRE).
% そのためにはalpha-geodesicsからのサンプリングが必要
% If sampling from alpha-geodesics is possible, the ordinary IMDRE framework can accomplish this, but sampling from these geodesics is non-trivial.
To do so along the geodesics, we propose to utilize the following importance sampling technique~\citep{tokdar2010importance}: consider the expectation
\begin{align}
    I \coloneqq \mathbb{E}_{\gamma^{(\alpha)}(\lambda)(\bm{x})}\left[g(\bm{x})\right] = \int_{\mathcal{X}}g(\bm{x})\gamma^{(\alpha)}(\lambda)(\bm{x})d\bm{x},
\end{align}
of any function $g(\bm{x})$ of a random variable $\bm{x} \sim \gamma^{(\alpha)}(\lambda)(\bm{x})$.
Let $\{\bm{x}_1,\dots,\bm{x}_T\}$ be the independent instances of sample size $T$, and the standard Monte Carlo integration is
\begin{align}
    \hat{I} \coloneqq \frac{1}{T}\sum^N_{i=1}g(\bm{x}_i).
\end{align}
From the law of large numbers, $\hat{I}$ converges to $I$ as $T\to\infty$.
Now, since it is difficult to sample directly from $\gamma^{(\alpha)}(\lambda)(\bm{x})$, considering another distribution $\pi(\bm{x})$, which allows for easy sampling, and we have
\begin{align}
    I = \int_{\mathcal{X}}g(\bm{x})\gamma^{(\alpha)}(\lambda)(\bm{x})d\bm{x} = \int_{\mathcal{X}}\frac{\gamma^{(\alpha)}(\lambda)(\bm{x})}{\pi(\bm{x})}g(\bm{x})\pi(\bm{x})d\bm{x} = \mathbb{E}_{\pi(\bm{x})}\left[\frac{\gamma^{(\alpha)}(\lambda)(\bm{x})}{\pi(\bm{x})}g(\bm{x})\right]. \label{eq:importance_sampling}
\end{align}
Then, let $w(\bm{x}) = \gamma^{(\alpha)}(\lambda)(\bm{x}) / \pi(\bm{x})$ and we have
\begin{align}
    \hat{I} = \frac{1}{T}\sum^T_{i=1} w(\bm{x}^*_i)g(\bm{x}^*_i) = \frac{1}{T}\sum^T_{i=1}\frac{\gamma^{(\alpha)}(\lambda)(\bm{x}^*_i)}{\pi(\bm{x}^*_i)}g(\bm{x}^*_i). 
\end{align}
where $\bm{x}^*_i \sim \pi(\bm{x})$.
Thus, the original problem falls into the problem of approximating the importance weighting $w(\bm{x})$.
Note that the normalization constant of the $\alpha$-geodesic is unknown, and we can consider the self-normalized importance sampling technique.
Then, let $\gamma^{(\alpha)}(\lambda)(\bm{x}) = q(\bm{x}) / (\int_{\mathcal{X}}q(\bm{x})d\bm{x}) \propto q(\bm{x})$ for $q \in \mathcal{M}$, and rewrite Eq.~\eqref{eq:importance_sampling} as
% \begin{align}
    $I = \int_{\mathcal{X}}g(\bm{x})\frac{q(\bm{x})}{\int_{\mathcal{X}}q(x)d\bm{x}} = \frac{\int_{\mathcal{X}}g(\bm{x})q(\bm{x})}{\int_{\mathcal{X}}q(\bm{x})d\bm{x}}$.
%\end{align}
Furthermore, if we assume that we are sampling from $p_s(\bm{x})$, we have
\begin{align}
    I = \frac{\int_{\mathcal{X}}g(\bm{x})\frac{q(\bm{x})}{p_s(\bm{x})}p_s(\bm{x})d\bm{x}}{\int_{\mathcal{X}}\frac{q(\bm{x})}{p_s(\bm{x})}p_s(\bm{x})d\bm{x}} = \frac{\mathbb{E}_{p_s(\bm{x})}\left[g(\bm{x})\frac{q(\bm{x})}{p_s(\bm{x})}\right]}{\mathbb{E}_{p_s(\bm{x})}\left[\frac{q(\bm{x})}{p_s(\bm{x})}\right]}. \label{eq:snis_expectation}
\end{align}
Applying importance sampling to the denominator and numerator of Eq.~\eqref{eq:snis_expectation} yields the following estimator.
\begin{align}
    \hat{I} = \frac{\frac{1}{T}\sum^T_{i=1}g(\bm{x}^*_i)w(\bm{x}^*_i)}{\frac{1}{T}\sum^T_{i=1}w(\bm{x}^*_i)} = \sum^T_{i=1}\frac{w(\bm{x}^*_i)}{\sum^T_{j=1}w(\bm{x}^*_j)}g(\bm{x}^*_i) = \sum^T_{i=1}\frac{\frac{\gamma^{(\alpha)}(\lambda)(\bm{x}^*_i)}{\pi(\bm{x}^*_i)}}{\sum^T_{j=1}\frac{\gamma^{(\alpha)}(\lambda)(\bm{x}^*_j)}{\pi(\bm{x}^*_j)}}g(\bm{x}^*_i).
\end{align}
% Proposition: alpha-geodesicsからのimportance sampling
In particular, suppose that the proxy distribution $\pi(\bm{x})=p_s(\bm{x})$ for importance sampling, sampling from the $\alpha$-geodesic connecting $p_s(\bm{x})$ and $p_t(\bm{x})$ can be accomplished by the following $w(\bm{x}^*)$.
\begin{align}
    w_\lambda(\bm{x}^*_i) = \frac{\gamma^{(\alpha)}(\lambda)(\bm{x}^*_i)}{p_s(\bm{x}^*_i)} = \left\{1 - \lambda + \lambda\left(\frac{p_t(\bm{x}^*_i)}{p_s(\bm{x}^*_i)}\right)^{\frac{1-\alpha}{2}}\right\}^{\frac{2}{1-\alpha}} = \left\{1 - \lambda + \lambda\left(\frac{1}{r(\bm{x}^*_i)}\right)^{\frac{1-\alpha}{2}}\right\}^{\frac{2}{1-\alpha}} \nonumber 
    % \hat{I} = \frac{1}{T}\sum^T_{i=1}\frac{\gamma^{(\alpha)}(\lambda)(\bm{x}^s_i)}{p_s(\bm{x}^s_i)}g(\bm{x}^s_i) &= \frac{1}{T}\sum^T_{i=1}\frac{\left\{(1-\lambda)p_s(\bm{x}^s_i)^{\frac{1-\alpha}{2}} + \lambda p_t(\bm{x}^s_i)^{\frac{1-\alpha}{2}}\right\}^{\frac{2}{1-\alpha}}}{p_s(\bm{x}^s_i)}g(\bm{x}^s_i) \nonumber \\
    % &= \frac{1}{T}\sum^T_{i=1}\left\{1 - \lambda + \lambda\left(\frac{p_t(\bm{x}^s_i)}{p_s(\bm{x}^s_i)}\right)^{\frac{1-\alpha}{2}}\right\}^{\frac{2}{1-\alpha}}g(\bm{x}^s_i) \nonumber \\
    % &= \frac{1}{T}\sum^T_{i=1}\left\{1 - \lambda + \lambda\left(\frac{1}{r(\bm{x}^s_i)}\right)^{\frac{1-\alpha}{2}}\right\}^{\frac{2}{1-\alpha}}g(\bm{x}^s_i).
\end{align}
% Therefore, the plug-in estimator is
% \begin{align}
%     \hat{I} = \frac{1}{T}\sum^T_{i=1}\frac{\left\{1 - \lambda + \lambda\left(1 / \hat{r}(\bm{x}^*_i)\right)^{\frac{1-\alpha}{2}}\right\}^{\frac{2}{1-\alpha}}}{\sum^T_{j=1}\left\{1 - \lambda + \lambda\left(1 / \hat{r}(\bm{x}^*_j)\right)^{\frac{1-\alpha}{2}}\right\}^{\frac{2}{1-\alpha}}}g(\bm{x}^*_i). \label{eq:importance_sampling_estimator}
% \end{align}
The key trick to obtain the density ratio $\gamma^{(\alpha)}(\lambda_i)(\bm{x}) / \gamma^{(\alpha)}(\lambda_j)(\bm{x})$ is the following cancellation out of the proxy distribution $\pi(\bm{x}) = p_s(\bm{x})$ using the ratio of the weighting functions $w_{\lambda_i}(\bm{x}) / w_{\lambda_j}(\bm{x})$.
\begin{align}
    \frac{w_{\lambda_i}(\bm{x})}{w_{\lambda_j}(\bm{x})} = \frac{\gamma^{(\alpha)}(\lambda_i)(\bm{x}) / p_s(\bm{x})}{\gamma^{(\alpha)}(\lambda_j)(\bm{x}) / p_s(\bm{x})} = \frac{\gamma^{(\alpha)}(\lambda_i)(\bm{x})}{\gamma^{(\alpha)}(\lambda_j)(\bm{x})}.
\end{align}
As a result, the ratio of the two distributions on the geodesic can be rewritten in a form that depends only on the density ratio $r(\bm{x})$.
Thus, if a weighting function is obtained for each $\lambda$, the density ratios between distributions on the geodesic can be recovered.
Note that the case $\pi(\bm{x})=p_t(\bm{x})$ can be computed almost similarly.
For example, if the function $g(\bm{x})$ is a loss function for a probabilistic classifier $f(\bm{x})$, as in logistic regression, with $y=+1$ for $\bm{x}\sim p_s$ and $y=-1$ for $\bm{x}\sim p_t$, the output learned by such a Monte Carlo integral approximation can be used for density ratio estimation.
%\begin{align}
    % \hat{r}(\bm{x}) = \frac{p_s(\bm{x})}{p_t(\bm{x})} &= \left(\frac{p(y = +1 \mid \bm{x})p(\bm{x})}{p(y=+1)}\right)\left(\frac{p(y = -1 \mid \bm{x})p(\bm{x})}{p(y = -1)}\right)^{-1} \nonumber \\
    % &= \frac{p(y = -1)}{p(y = +1)}\frac{p(y = +1 \mid \bm{x})}{p(y = -1 \mid \bm{x})} \approx \frac{n_t}{n_s}\frac{f(\bm{x})}{1 - f(\bm{x})}. \label{eq:dre_probabilistic_classifier}
%    \hat{r}(\bm{x}) = \left(\frac{p(y = +1 \mid \bm{x})p(\bm{x})}{p(y=+1)}\right)\left(\frac{p(y = -1 \mid \bm{x})p(\bm{x})}{p(y = -1)}\right)^{-1} \approx \frac{n_t}{n_s}\frac{f(\bm{x})}{1 - f(\bm{x})}. \label{eq:dre_probabilistic_classifier}
%\end{align}

However, for the importance sampling, it is necessary to obtain an importance weighting function $w_\lambda(\bm{x})$ that depends on the estimate of the density ratio $r(\bm{x})$.
In other words, the density ratio estimator and the weighting function for sampling are interdependent and deadlocked.

\begin{table}[t]
    \centering
    \caption{Evaluation results at different step sizes $m$ with $\alpha=3$.
    The evaluation metric is the mean absolute error (MAE) between $\hat{r}$ and $r$. The means and standard deviations of 10 trials are reported.}
    \resizebox{\columnwidth}{!}{
    \begin{tabular}{cccccccc}
        \toprule
         & $m=10$ & $m=20$ & $m=30$ & $m=40$ & $m=50$ & $m=70$ & $m=100$ \\ \midrule
         & $35.41 (\pm 3.55)$ & $34.80 (\pm 3.09)$ & $34.61 (\pm 2.60)$ & $34.32 (\pm 2.54)$ & $34.26 (\pm 2.47)$ & $34.18 (\pm 2.39)$ & $34.13 (\pm 2.34)$\\
         \bottomrule
    \end{tabular}
    }
    \label{tab:gaussian_n_steps}
\end{table}
\begin{figure}[t]
    \centering
    \includegraphics[width=\linewidth]{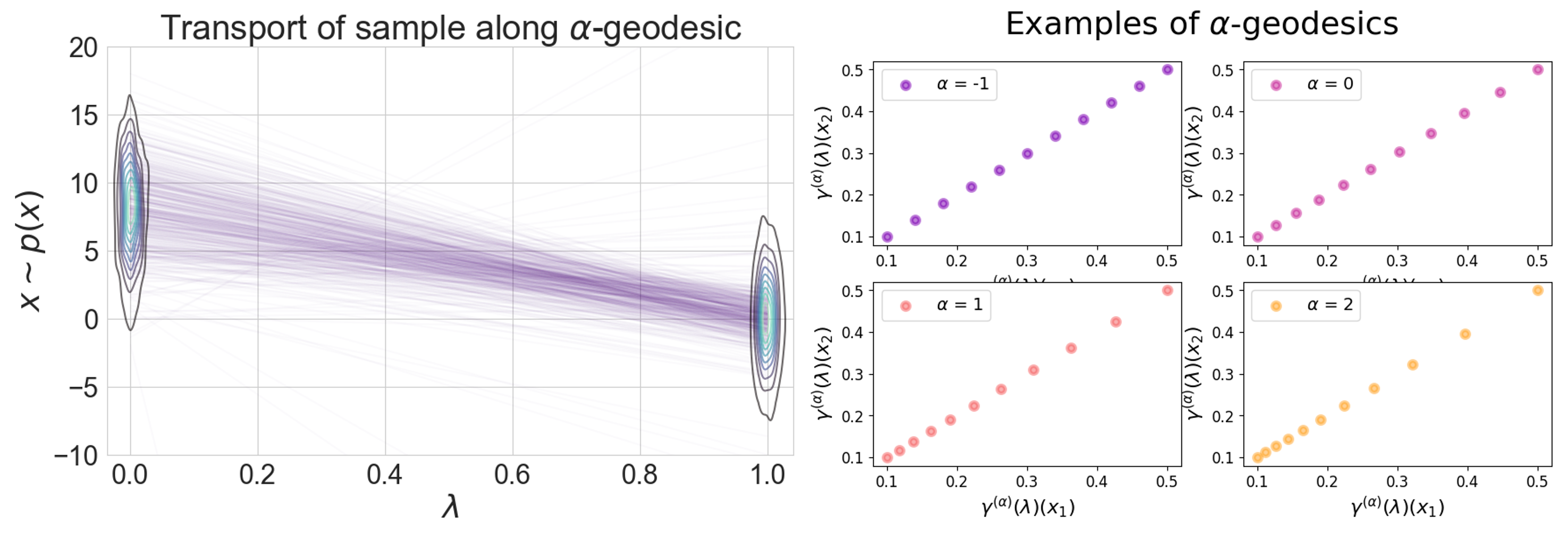}
    \caption{Illustrative examples of GIMDRE convergence. Left panel: an example of sample transport from $p_s$ to $p_t$ along $\alpha$-geodesic. Right panel: examples of $\alpha$-geodesics connecting $p_s(\bm{x})$ and $p_t(\bm{x})$ in two dimensional case.
    Here, we assume that $p_s(\bm{x}) = (0.1, 0.1)$ and $p_t(\bm{x}) = (0.9, 0.9).$}
    \label{fig:toy_example_gaussian}
\end{figure}

% 交互最適化アルゴリズムを提案
To address this problem, we develop the following alternating optimization algorithm.
\begin{itemize}
    \item[i)] The density ratio estimator $\hat{r}(\bm{x})$ is estimated from the samples $\{\bm{x}_i\}^{n_s}_{i=1}$ and $\{\bm{x}_i\}^{n_t}_{i=1}$.
    \item[ii)] Obtained $\hat{r}(\bm{x})$ is used to update the weighting $w_\lambda(\bm{x})$ for the importance sampling.
    \item[iii)] Using the updated weighting $w_\lambda(\bm{x})$, $\hat{r}(\bm{x})$ is estimated by GIMDRE along the $\alpha$-geodesic.
    \item[iv)] Iterate steps ii) and iii) over a predetermined number of times.
\end{itemize}
Figure~\ref{fig:alternating_algorithm} shows an overview of the above algorithm.
The plug-in estimator that underlies our algorithm can be expected to have different statistical behavior depending on two parameters: $\alpha$, which determines the shape of the curve, and $\lambda$, which determines the position on the curve.

\subsection{Equally Spaced Transitions on Generalized Geodesic}
The right panel of Figure~\ref{fig:toy_example_gaussian} visualizes the trajectory approaching from the source point to the target point along several $\alpha$-geodesics (see Section~\ref{sec:numerical_experiments} for more details).
In this figure, we have equally spaced transitions of $\lambda$ from $0$ to $1$ in increments of $0.1$.
As another way of selecting $\lambda$, we can consider selecting equally spaced transitions on the geodesic.
That is, we can choose $\lambda_1,\dots,\lambda_m$ such that the arc lengths $\|\gamma^{(\alpha)}(\lambda_i)\|,\dots,\|\gamma^{(\alpha)}(\lambda_m)\|$ of the $\alpha$-geodesic at time $\lambda_1,\dots,\lambda_m$ are all equal.
Let $\gamma^{(\alpha)}(t)$ be the $\alpha$-geodesic.
The length of this curve connecting $\gamma^{(\alpha)}(0) = p_s$ and $\gamma^{(\alpha)}(1) = p_t$ is defined as
\begin{align}
    \|\gamma^{(\alpha)}\| \coloneqq \int^1_0 \sqrt{\sum_{i,j}g_{ij}\frac{d\gamma^{(\alpha)i}(t)}{dt}\frac{d\gamma^{(\alpha)j}(t)}{dt}},
\end{align}
where $\gamma^{(\alpha)i}$ is the $i$-th component of $\gamma^{(\alpha)}$.
Then, we can see that equally spaced transitions $\lambda_1,\dots,\lambda_m$ on the $\alpha$-geodesic satisfy
\begin{align}
    \int^{\lambda_k}_{\lambda_{k-1}} \sqrt{\sum_{i,j}g_{ij}\frac{d\gamma^{(\alpha)i}(t)}{dt}\frac{d\gamma^{(\alpha)j}(t)}{dt}} = \frac{\|\gamma^{(\alpha)}\|}{m}, \quad 1 \leq k \leq m.
\end{align}
As can be seen from the right panel of Figure~\ref{fig:toy_example_gaussian}, there is a large jump in the approaching distribution for large $\alpha$.
Therefore, adjusting the transition speed of $\lambda$ may be useful in practical applications.

\section{Numerical Experiments}
\label{sec:numerical_experiments}
In this section, we design numerical experiments to investigate the behavior of GIMDRE and report the results.
Our series of numerical experiments consists of two parts: i) illustrative examples to clarify GIMDRE behavior, and ii) experiments on the two-sample test, a major application of density ratio in statistics.
See Appendix~\ref{apd:experiments} for more details on the experiments.

\subsection{Illustrative Examples}
\label{sec:illustrative_examples}
First, we consider illustrative examples to clarify the behavior of GIMDRE.
% このセクションでの実験の目的は以下を確かめることである：i) GIMDREのための交互最適化アルゴリズムは収束し，良い推定を与えるのか？ ii) サンプルサイズ，次元数および二つの分布の離れ度合いを変えた時，どのような測地線に沿った推定が良いか？ iii) lambdaの分割数が性能にどのように影響を及ぼすか？
The purpose of the experiments in this section is to verify the following: does the alternating optimization algorithm for GIMDRE converge and give good estimates, and what geodesics are better to estimate along when varying the sample size, the number of dimensions, and the distance between the two distributions.
We generate samples $\{\bm{x}_i^s\}^{n_s}_{i=1}, \{\bm{x}_i^t\}^{n_t}_{i=1}$ of sample sizes $n_s$ and $n_t$, respectively, from multivariate Gaussian distributions $\mathcal{N}(\bm{\mu}_s, \bm{\Sigma}_s), \mathcal{N}(\bm{\mu}_t, \bm{\Sigma}_t)$ with $\bm{\mu}_s, \bm{\mu}_t \in \mathbb{R}^d$ and $\bm{\Sigma}_s, \bm{\Sigma}_t \in \mathbb{R}^{d\times d}$.
The values of these parameters are specified specifically for each experiment.

\begin{figure}
    \centering
    \includegraphics[width=\linewidth]{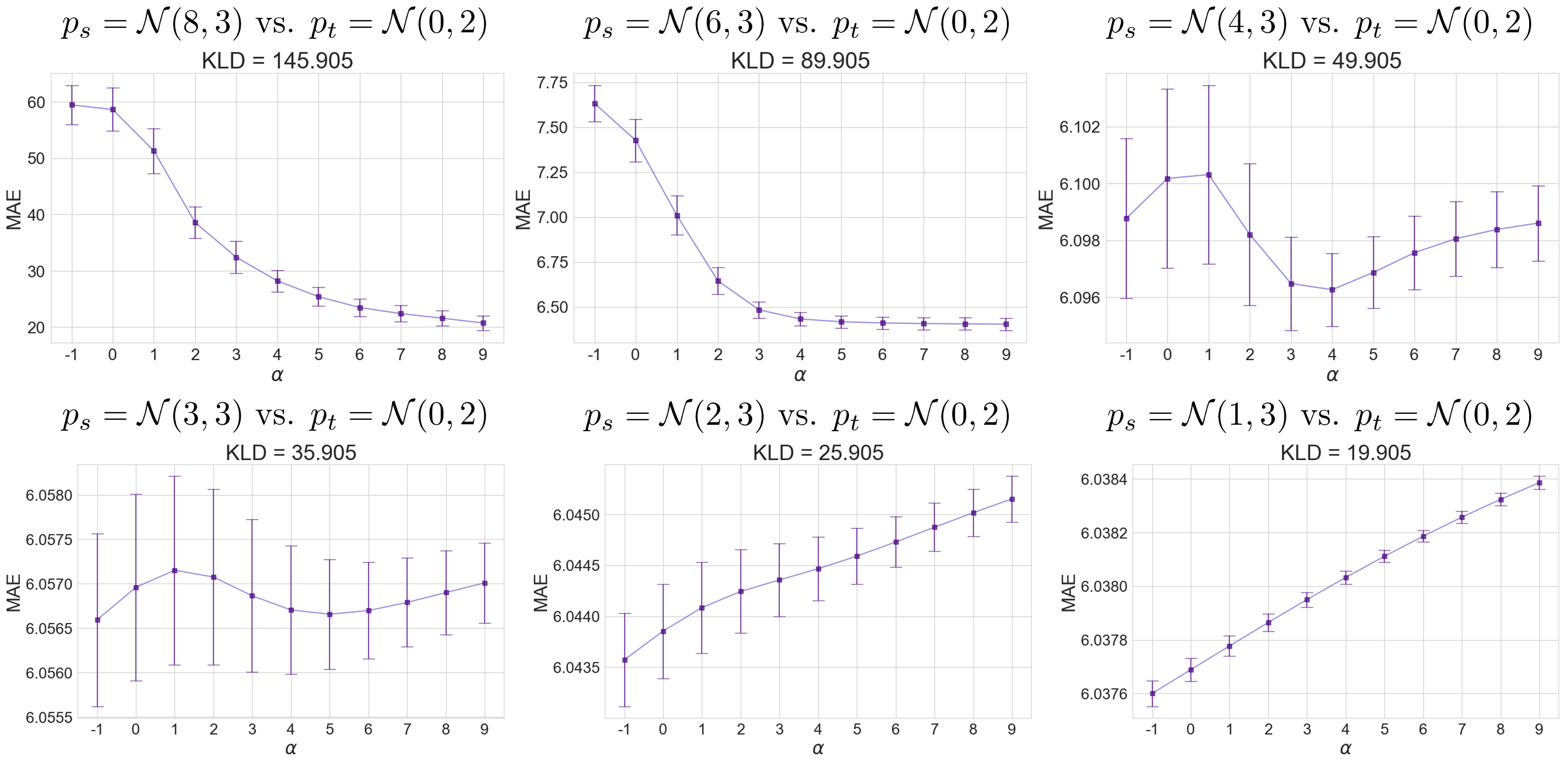}
    \caption{Behavior of GIMDRE according to $\alpha$ when two distributions are brought closer together. KLD is the analytically calculated KL-divergence value}
    \label{fig:alpha_vs_kld}
\end{figure}

First, for illustrative purposes, we consider the case where $\alpha=3$ and the samples follow univariate Gaussian distributions as $p_s = \mathcal{N}(8, 3)$ and $p_t = \mathcal{N}(0, 2)$.
The sample sizes are $n_s = n_t = 500$.
The base method for DRE is based on a probabilistic classifier using logistic regression.
This setup is a typical example where classical DRE fails because the two distributions are far apart.
That is, it is known that classical DRE is difficult unless the means of the two distributions are somewhat close and the variance of the distribution on the denominator side is sufficiently large compared to that on the numerator side.
In fact, the mean absolute error (MAE) obtained by ordinary DRE with the base method is $264.07 (\pm 116)$.
Note that hereafter we will report the mean and standard deviation of the $10$ trials unless otherwise specified.
Table~\ref{tab:gaussian_n_steps} shows the results of GIMDRE evaluations when different step sizes are used.
These results show that even with a small number of steps, there is a significant improvement compared to the base method, and that increasing the number of steps further improves the mean and standard deviation of the evaluation results.
The left panel of Figure~\ref{fig:toy_example_gaussian} also shows an illustration of sample transport using the estimated density ratio $\hat{r}$.
This figure shows that $\hat{r}$ is well able to transport from $p_s = \mathcal{N}(8, 3)$ to $\mathcal{N}(0, 2)$.
In addition, the right panel of Figure~\ref{fig:toy_example_gaussian} shows examples of $\alpha$-geodesics in two dimensional case.
In this figure, $\lambda \in \{0, 0.1, 0.2, 0.3, 0.4, 0.5, 0.6, 0.7, 0.8, 0.9, 1.0\}$ for all $\alpha$.
It can be seen that for $\alpha = -1$, one approaches the target distribution at a uniform rate as the $\lambda$ increases.

Next, Figure~\ref{fig:alpha_vs_kld} shows an illustration of the behavior of GIMDRE according to $\alpha$ when two distributions are brought together.
The parameter $\mu_s \in \{8, 6, 4, 3, 2, 1\}$, and other parameters are fixed as $m=100$, $n = n_s = n_t = 500$, $p_s = \mathcal{N}(\mu_s, 3)$ and $p_t = \mathcal{N}(0, 2)$.
To quantify the discrepancy between the two distributions, we analytically compute the KL-divergence.
This result shows that the use of a large $\alpha$ consistently suppresses the variance of the estimator.
In terms of bias, a large $\alpha$ is effective when the KL-divergence between the two distributions is large, while the effect is reversed as the distributions get closer together.
Tables~\ref{tab:sample_size_alpha} and~\ref{tab:dimension_alpha} show the evaluation results for different sample sizes and sample dimensions.
It can be seen that the degradation of the estimator with decreasing sample size and increasing dimensionality is reduced by using a larger $\alpha$.

\begin{table}[t]
    \caption{Evaluation results with different $\alpha$ and sample sizes $n = n_s = n_t$.}
    \centering
    \resizebox{\columnwidth}{!}{
    \begin{tabular}{lccccc}
        \toprule
                     & $n=100$ & $n=200$ & $n=300$ & $n=400$ & $n=500$ \\
         \midrule
         $\alpha=-1$ & $180.06 (\pm 48.14)$ & $126.48 (\pm 25.52)$ & $80.67 (\pm 16.16)$ & $64.30 (\pm 7.21)$ & $59.88 (\pm 4.50)$\\
         $\alpha=3$  & $76.39 (\pm 20.77) $ & $54.99 (\pm 14.23)$ & $43.60 (\pm 7.15)$ & $36.36 (\pm 3.59)$ & $34.13 (\pm 2.34)$ \\
         $\alpha=7$  & $69.14 (\pm 18.51)$ &  $41.26 (\pm 12.75)$ & $35.33 (\pm 5.62)$ & $26.11 (\pm 2.506)$ & $23.27 (\pm 1.86)$\\
         \bottomrule
    \end{tabular}
    }
    \label{tab:sample_size_alpha}
\end{table}

From the above evaluation results, we can make the following suggestions: i) if the two distributions are close enough, a good estimator can be obtained with a small $\alpha$, and as the two distributions move away from each other, the usefulness of using a large $\alpha$ increases. In particular, a large $\alpha$ consistently leads to a variance reduction in the estimator.
% \begin{itemize}
%     \item If the two distributions are close enough, a good estimator can be obtained with a small $\alpha$.
%     \item As the two distributions move away from each other, the usefulness of using a large $\alpha$ increases. In particular, a large $\alpha$ consistently leads to a variance reduction in the estimator.
% \end{itemize}

\begin{table}[t]
    \caption{Evaluation results with different $\alpha$ and sample dimensions $d$.}
    \resizebox{\columnwidth}{!}{
    \centering
    \begin{tabular}{lccccc}
        \toprule
                     & $d=2$ & $d=3$ & $d=4$ & $d=5$ & \\
        \midrule
         $\alpha=-1$ & $260.55 (\pm 58.23)$ &  $589.62 (\pm 103.06)$ &  $6030.19 (\pm 4845.87)$ &  $16337.70 (\pm 9379.11)$ & \\
         $\alpha=3$  & $123.83 (\pm 18.23)$ &  $320.65 (\pm 28.31)$ &  $504.50 (\pm 47.42)$ &  $3664.92 (\pm 287.02)$ & \\
         $\alpha=7$  & $122.74 (\pm 16.17)$ &  $331.91 (\pm 24.26)$ &  $692.35 (\pm 44.36)$ & $3792.96 (\pm 283.45)$ & \\
         \bottomrule
    \end{tabular}
    }
    \label{tab:dimension_alpha}
\end{table}

\subsection{Two-Sample Test based on GIMDRE}
\label{sec:two_sample_test}
In statistics, given two sets of samples, testing whether the probability distributions behind the samples are equivalent is a fundamental task.
This problem is referred to as the two-sample test or homogeneity test~\citep{stein1945two,conover1981comparative,gretton2006kernel}.
In our experiments, we utilize the two-sample test based on the permutation test~\citep{efron1993permutation,sugiyama2011least}.
Let $X^s$ and $X^t$ be two samples from $p_s$ and $p_t$.
We first estimate the Pearson divergence using $X^s$ and $X^t$ as
 \begin{align}
    \hat{D}_{\mathrm{PE}}[X^s \| X^t] \coloneqq \frac{1}{2n_s}\sum^{n_s}_{i=1}\hat{r}(\bm{x}^s_i) - \frac{1}{n_t}\sum^{n_t}_{i=1}\hat{r}(\bm{x}^t_i) + \frac{1}{2}.
\end{align}
Next, we randomly permute the $\left|X^s \cup X^t\right|$ samples, and assign the first $n_s$ samples to a set $\tilde{X}$ and the remaining $n_t$ samples to another set $\tilde{X}'$.
Then, we estimate the Pearson divergence again using the randomly shuffled samples $\tilde{X}$ and $\tilde{X}'$, and we obtain an estimate $\hat{D}_{\mathrm{PE}}[\tilde{X} \| \tilde{X}']$.
This random shuffling procedure is repeated $K$ times, and the distribution of $\hat{D}_{\mathrm{PE}}[\tilde{X} \| \tilde{X}']$ under the null hypothesis (that is, the two distributions are the same) is constructed.
In this experiment, we set $K=100$.
Finally, the $p$-value is approximated by evaluating the relative ranking of $\hat{D}_{\mathrm{PE}}[X^s \| X^t]$ in the distribution of $\hat{D}_{\mathrm{PE}}[\tilde{X} \| \tilde{X}']$.
Suppose that $n_s = n_t$, and let $F$ be the distribution function of $\hat{D}_{\mathrm{PE}}[\tilde{X} \| \tilde{X}']$.
Let $\beta \coloneqq \sup \{t \in \mathbb{R} \mid F(t) \leq 1 - \alpha\}$ be the upper $100\alpha$-percentile point of $F$.
If $p_s = p_t$, it was shown that $\mathbb{P}(\hat{D}_{\mathrm{PE}}[X^s \| X^t] > \beta) \leq \alpha$.
That is, for a given significance level $\alpha$, the probability that $\hat{D}[X^s \| X^t]$ exceeds $\beta$ is at most $\alpha$ when $p_s = p_t$.
Thus, when the null hypothesis is correct, it will be properly accepted with a specified probability.
Let $n = n_s = n_t = 500$, and $p_s = \mathcal{N}(0, 1)$.
We consider the setups as $\mu_t \in \{0, 0.5, 1\}$ and $\sigma_t \in \{1, 1.5\}$.
The histogram plots in Figure~\ref{fig:ptest} show the distribution of estimated $\hat{D}_{\mathrm{PE}}[\tilde{X} \| \tilde{X}']$, with red crosses indicating $\hat{D}_{PE}[X^s \| X^t]$.
Also, the line plots show the transition of the $p$-values with increasing sample size in the different settings.
For all the alphas, when $p_s \neq p_t$, the $p$-value gets smaller with each increase in sample size, while if the two distributions are equal, it always produces a large $p$-value.
The figure also shows that the GIMDRE with large $\alpha$ behaves conservatively.

\begin{figure}
    \centering
    \includegraphics[width=\linewidth]{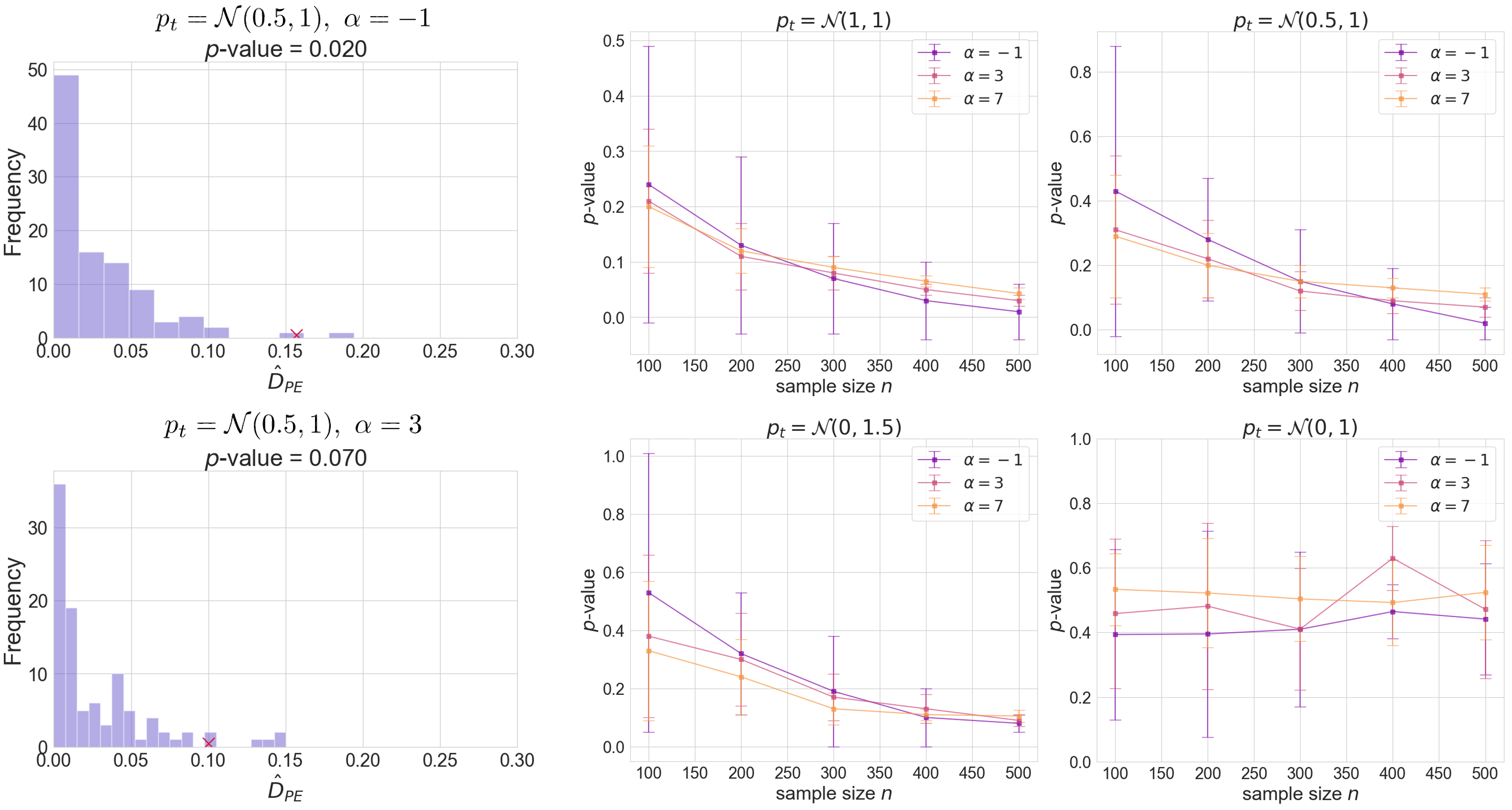}
    \caption{Experimental results of two-sample test. The histogram plots show the distribution of estimated $\hat{D}_{\mathrm{PE}}[\tilde{X} \| \tilde{X}']$, with red crosses indicating $\hat{D}_{PE}[X^s \| X^t]$. The line plots show the transition of the $p$-values with increasing sample size in the different settings.}
    \label{fig:ptest}
\end{figure}

\section{Conclusion and Discussion}
\label{sec:conclusion}
In this paper, we first geometrically reinterpreted the framework of incremental density ratio estimation using a mixture between two distributions.
The effectiveness of such a framework for density ratio estimation has been reported in recent years, and it is very useful to study this behavior in detail.
Our analysis demonstrated that considering what kind of bridge distributions to create is equivalent to choosing what kind of curve on a statistical manifold.
Such a geometrically intuitive interpretation not only reveals properties of the algorithm, but also induces natural generalizations.
Next, we considered density ratio estimation along arbitrary generalized geodesics on the manifold.
This procedure requires sampling from curves on the manifold, called $\alpha$-geodesics, which is non-trivial.
This is because the two probability distributions are generally not given explicitly.
Here we have shown that by using the importance sampling framework, this sampling can be written in a form that depends on the density ratio that we originally wanted to obtain.
This is just the state of interdependence, and we have demonstrated that this deadlock can be resolved with a simple alternating optimization procedure.
We have shown that this strategy works through illustrative examples and experiments on hypothesis testing, an important task in statistics.
Limitlations and broader impacts of this study are discussed in Appendix~\ref{apd:limitations_broader_impact}.
We believe that this study provides useful insights from a geometric perspective into a fundamental task in statistics.

% \begin{ack}
% TBA
% \end{ack}

\clearpage
{
\small
% \bibliographystyle{abbrvnat}
% \bibliography{neurips_2024}

}

%%%%%%%%%%%%%%%%%%%%%%%%%%%%%%%%%%%%%%%%%%%%%%%%%%%%%%%%%%%%

\clearpage
\appendix
\addcontentsline{toc}{section}{Appendix} % Add the appendix text to the document TOC
\part{Appendix} % Start the appendix part
\parttoc

\clearpage

\section{Limitations and Broader Impact}
\label{apd:limitations_broader_impact}
In this section, we discuss about both the limitations and the broader impacts of our study.
\subsection{Limitations}
One limitation is that our algorithm does not include a neural network.
Of course, on the other hand, this is a promising extension direction.
Modification of our algorithm to take advantage of the expressive power of deep learning could lead to further advances in this research.
Deep learning models have demonstrated remarkable capabilities in capturing complex patterns and relationships in data, offering a prospect for enhancing the performance and scope of our algorithm.
On the other hand, an existing study reported that the effect of importance-weighted learning by deep learning decays with the number of steps~\citep{byrd2019effect}, so it is not clear whether our method based on importance sampling is straightforwardly extendable or not.
For the application of neural networks to our research, we need to take into account such negative results reported by existing studies, and it seems that we need to devote a large space to this.

\subsection{Broader Impact}
As highlighted in the Introduction section, the density ratio stands out as a fundamental tool in statistics, playing a pivotal role in various fields.
Therefore, while our algorithm itself does not directly produce social impacts, its downstream applications may, depending on how it is used.
One example is the observation in Section~\ref{sec:numerical_experiments}, where the choice of parameters leads to conservative hypothesis tests.
How sensitive a hypothesis test should be depends on whether it is used in a life-threatening setting, such as the medical field, or a place where challenging decisions are allowed, such as marketing~\citep{dudoit2003multiple,lafond2008information,lee2018proper,storey2004strong}.
Our numerical experiments suggest that the choice of parameters can control how sensitive a hypothetical test is constructed, which could have a sufficient impact on society.

\clearpage

\section{Technical Details}
\label{apd:technical_details}
In this section, we provide supplementary technical details, including proofs of the theorems and additional discussion of the ESS.

\subsection{Proofs of Theoretical Results}
\label{apd:proofs}

\begin{proof}[Proof of Theorem~\ref{thm:linear_imdre}]
    Recall that $\alpha$-geodesics between $p$ and $q$ are given as
    \begin{align}
        \gamma^{(\alpha)}(\lambda) = \left\{(1-\lambda)p(\bm{x})^{\frac{1-\alpha}{2}} + \lambda q(\bm{x})^{\frac{1-\alpha}{2}}\right\}^{\frac{2}{1-\alpha}}.
    \end{align}
    Substituting $\alpha=-1$, we see that this is an ordinary weighted average.
    Also, $\alpha$-divergence becomes KL-divergence with $\alpha=-1$, thus obtaining the proof.
\end{proof}

\begin{proof}[Proof of Theorem~\ref{thm:asymptotic_expansion}]
    Considering the following Taylor expansion.
    \begin{align}
        \log\frac{p_s(x)}{\hat{r}(x)p_t(x)} &= \log\frac{r(x)}{\hat{r}(x)} \nonumber \\
        &= -\log\frac{\hat{r}(x)}{r(x)} \nonumber \\
        &= -\left(\frac{\hat{r}(x)}{r(x)} - 1\right) + \frac{1}{2}\left(\frac{\hat{r}(x)}{r(x)} - 1\right)^2 + O_p\left(\left|\frac{\hat{r}(x)}{r(x)} - 1\right|^3\right).
    \end{align}
    Let
    \begin{align}
        J_1 &= \int_{\mathcal{X}}p_s(x)\log\frac{p_s(x)}{\hat{r}(x)p_t(x)}dx - 1, \\
        J_2 &= \int_{\mathcal{X}}p_t(x)\hat{r}(x)dx.
    \end{align}
    Then,
    \begin{align}
        J_1 &= \int_{\mathcal{X}}p(x)\log\frac{p_s(x)}{\hat{r}(x)p_t(x)}dx - 1 \nonumber \\
        &= \int_{\mathcal{X}} \left\{-\left(\frac{\hat{r}(x)}{r(x)} - 1\right) + \frac{1}{2}\left(\frac{\hat{r}(x)}{r(x)} - 1\right)^2\right\}p_s(x)dx - 1 + O(\|\hat{r}/r - 1\|^3)\nonumber \\
        &= \int_{\mathcal{X}}\left\{-\frac{\hat{r}(x)}{r(x)} + 1 + \frac{1}{2}\left(\left(\frac{\hat{r}(x)}{r(x)}\right)^2 - 2\frac{\hat{r}(x)}{r(x)} + 1\right)\right\}p_s(x)dx - \int_{\mathcal{X}}p_s(x)dx + O(\|\hat{r}/r - 1\|^3) \nonumber \\
        &= \int_{\mathcal{X}}\left\{-\frac{\hat{r}(x)}{r(x)} + \frac{1}{2}\left(\left(\frac{\hat{r}(x)}{r(x)}\right)^2 - 2\frac{\hat{r}(x)}{r(x)} + 1\right)\right\}p_s(x)dx + O(\|\hat{r}/r - 1\|^3) \nonumber \\
        &= \int_{\mathcal{X}}\left\{-\frac{\hat{r}(x)}{r(x)} + \frac{1}{2}\frac{\hat{r}(x)^2}{r(x)^2} - \frac{\hat{r}(x)}{r(x)} + \frac{1}{2}\right\}p_s(x)dx + O(\|\hat{r}/r - 1\|^3)\nonumber \\
        &= \int_{\mathcal{X}}\left\{-2\frac{\hat{r}(x)}{r(x)} + \frac{1}{2}\frac{\hat{r}(x)^2}{r(x)^2} + \frac{1}{2}\right\}p_s(x)dx + O(\|\hat{r}/r - 1\|^3) \nonumber \\
        &= \int_{\mathcal{X}}\left\{-2\hat{r}(x) + \frac{1}{2}\frac{\hat{r}(x)^2}{r(x)} + \frac{1}{2}r(x)\right\}\frac{1}{r(x)}p_s(x)dx + O(\|\hat{r}/r - 1\|^3) \nonumber \\
        &= \int_{\mathcal{X}}\left\{-2\hat{r}(x) + \frac{1}{2}\frac{\hat{r}(x)^2p_t(x)}{p_s(x)} + \frac{1}{2}\frac{p_s(x)}{p_t(x)}\right\}p_t(x)dx + O(\|\hat{r}/r - 1\|^3).
    \end{align}
    Here,
    \begin{align}
        \|\hat{r} / r - 1\| \coloneqq \left(\int p_s(x)|\hat{r}(x) / r(x) - 1|^2 dx\right)^{1/2}.
    \end{align}
    We have
    \begin{align}
        D_{\mathrm{KL}}[p_s \| \hat{r}\cdot p_t] &= J_1 + J_2 \nonumber \\
        &= \int_{\mathcal{X}}\left\{-2\hat{r}(x) + \frac{1}{2}\frac{\hat{r}(x)^2p_t(x)}{p_s(x)} + \frac{1}{2}\frac{p_s(x)}{p_t(x)}\right\}p_t(x)dx + \int_{\mathcal{X}} p_t(x)\hat{r}(x)dx \nonumber \\
        &= \int_{\mathcal{X}}\left\{-\hat{r}(x) + \frac{1}{2}\frac{\hat{r}(x)^2p_t(x)}{p_s(x)} + \frac{1}{2}\frac{p_s(x)}{p_t(x)}\right\}p_t(x)dx \nonumber \\
        &= \frac{1}{2}\int_{\mathcal{X}}\left\{-2\hat{r}(x) + \frac{\hat{r}(x)^2p_t(x)}{p_s(x)} + \frac{p_s(x)}{p_t(x)}\right\}p_t(x)dx \nonumber \\
        &= \frac{1}{2}\int_{\mathcal{X}}\frac{1}{p_s(x)}\left\{-2p_s(x)\cdot\hat{r}(x)p_t(x) + \hat{r}(x)^2p_t(x)^2 + p_s(x)^2\right\}dx \nonumber \\
        &= \frac{1}{2}\int_{\mathcal{X}}\frac{(p_s(x) - \hat{r}(x)p_t(x))^2}{p_s(x)}dx + O(\|\hat{r} / r - 1\|^3) \nonumber \\
        &= D_{\mathrm{PE}}[p_s \| \hat{r}(x)p_t(x)] + O(\|\hat{r}/r - 1\|^3).
    \end{align}
    Here, $D_{\mathrm{PE}}[\cdot \| \cdot]$ is the Pearson divergence.
    Hence, the asymptotic expansion of the unnormalized KL-divergence between $p_s(x)$ and $\hat{r}(x)p_t(x)$ is given as
    \begin{align}
        D_{\mathrm{KL}}[p_s \| \hat{r}\cdot p_t] = D_{\mathrm{PE}}[p_s \| \hat{r}\cdot p_t] + O(n^{-3/2}).
    \end{align}
    We consider the following $\alpha$-divergence.
    \begin{align}
        D_\alpha[p_s \| p_t] = \frac{1}{\alpha(\alpha - 1)}\left( 1- \int p_s(x)^\alpha p_t(x)^{1-\alpha}dx\right).
    \end{align}
    From the simple calculation, we can confirm that Pearson divergence is a special case of $\alpha$-divergene with $\alpha = 2$.
    Since $D_{\mathrm{KL}}[p_s \| p_t]$ and $D_\alpha[p_s \| p_t]$ are minimizers of $\alpha$-geodesics with $\alpha=-1$ and $\alpha=2$ between $p_s$ and $p_t$, we have the proof.
\end{proof}

\begin{proof}[Proof of Corollary~\ref{thm:asymptotic_hellinger}]
   We assume that two distributions $p_s$ and $p_t$ are mutually absolutely continuous, and satisfy $0 < c_0 \leq dp_s / dp_t \leq c_1$ on the support of $p_s$ and $p_t$.
   We also assume that for any function $f\colon \to \mathbb{R}$, there exist $\epsilon$ and $\delta$ such that $\mathbb{E}_{p_t}[f(x)] \geq \epsilon, \ \|f\|_\infty \leq \delta$.
   Finally, we assume that for some constants $0 < \zeta < 2$ and $M$, $\sup_{q\in\mathcal{M}}\ln N(\epsilon, F^M, L_2(q))$, where $F$ is a set of finite linear combinations of integrable functions $f$ with positive coefficients bounded above by $M$, and $N(\epsilon, F^M, L_2(q))$ is the $\epsilon$-covering number with $L_2$-distance~\citep{nguyen2007estimating}.
   From the definition of $\zeta_n$,
   \begin{align}
       -\frac{1}{n_s}\sum^{n_2}_{i=1}\ln \hat{r}(\bm{x}^s_i) \leq -\frac{1}{n_s}\sum^{n_s}_{i=1}\ln \frac{r(\bm{x}^s_i)}{\frac{1}{n^t_i}\sum^{n_t}_{i=1}r(\bm{x}^t_i)} + \zeta_n.
   \end{align}
   By convexity of $-\ln x$,
   \begin{align}
       -\frac{1}{n_s}\sum^{n_s}_{i=1}\ln\left(\frac{\hat{r}(\bm{x}^s_i) + r(\bm{x}^s_i)\frac{1}{n^t_i}\sum^{n_t}_{i=1}r(\bm{x}^t_i)}{2r(\bm{x}^s_i)\frac{1}{n^t_i}\sum^{n_t}_{i=1}r(\bm{x}^t_i)}\right) &\leq \frac{\zeta_n}{2}.
   \end{align}
   Let $\xi_1 = \hat{r}(\bm{x}^s_i)$ and $\xi_2 = 2r(\bm{x}^s_i)\frac{1}{n^t_i}\sum^{n_t}_{i=1}r(\bm{x}^t_i)$, and we have
   \begin{align}
       (\hat{p}_t - p_t)(\xi_2 - \xi_1) - (\hat{p}_s - p_s)\ln\frac{\xi_2}{\xi_1} - \frac{\zeta_n}{2} &\leq -p_t(\xi_2 - \xi_1) - p_s\left(\ln\frac{\xi_2}{\xi_1}\right) \nonumber \\
       &\leq 2p_s\left(\sqrt{\frac{\xi_2}{\xi_1}} - 1\right) - p_t(\xi_2 - \xi_1) \nonumber \\
       &= p_t\left(2\sqrt{\xi_2\xi_1} - \xi_2 - \xi_1\right).
   \end{align}
   Here, the last line is the generalized version of Hellinger distance, and it corresponds to $\alpha$-divergence with $\alpha=0$.
   The rest of the proof follows immediately from \citet{sugiyama2008direct}.
\end{proof}

\begin{figure}
    \centering
    \includegraphics[width=\linewidth]{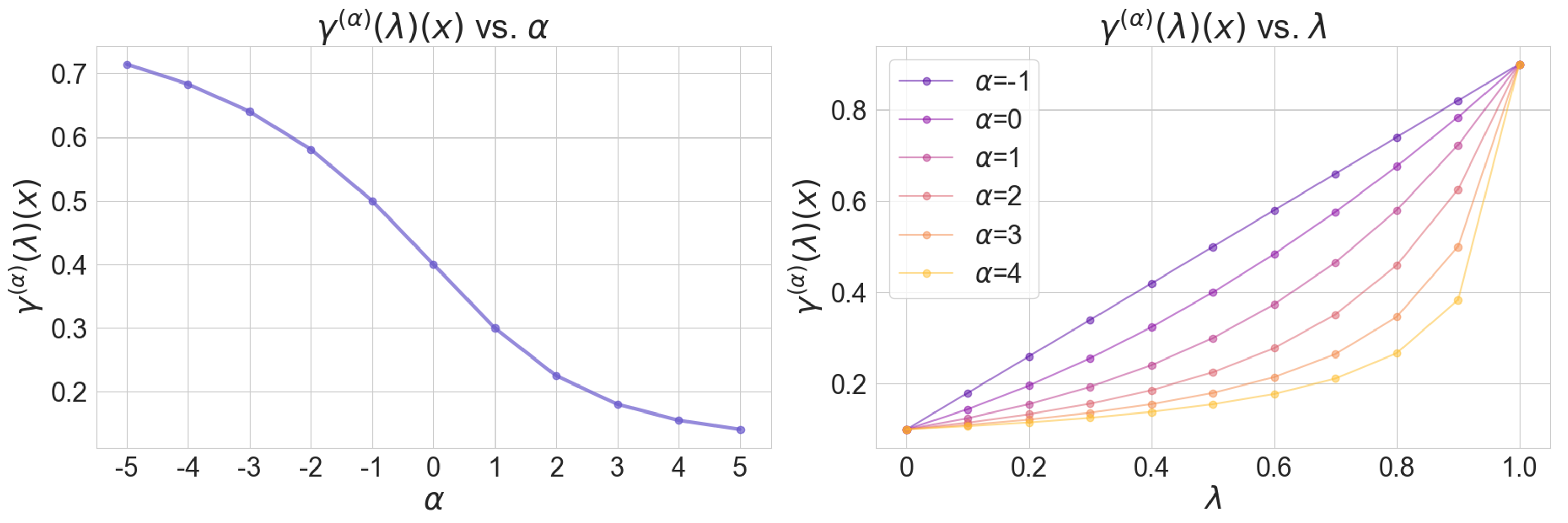}
    \caption{Monotonicity of $\alpha$-geodesics. The left panel shows that the value of the $\alpha$-geodesic decreases monotonically as $\alpha$ is increased. Here, $p_s(\bm{x})=0.1$, $p_t(\bm{x})=0.1$, and $\lambda$ is set to a constant of $0.5$. The right panel shows that the $\alpha$-geodesic approaches $p_s$ at $\lambda\to 0$ and $p_t$ at $\lambda\to 1$ for all $\alpha$.}
    \label{fig:gamma_monotonically}
\end{figure}

\begin{figure}[t]
    \centering
    \includegraphics[width=\linewidth]{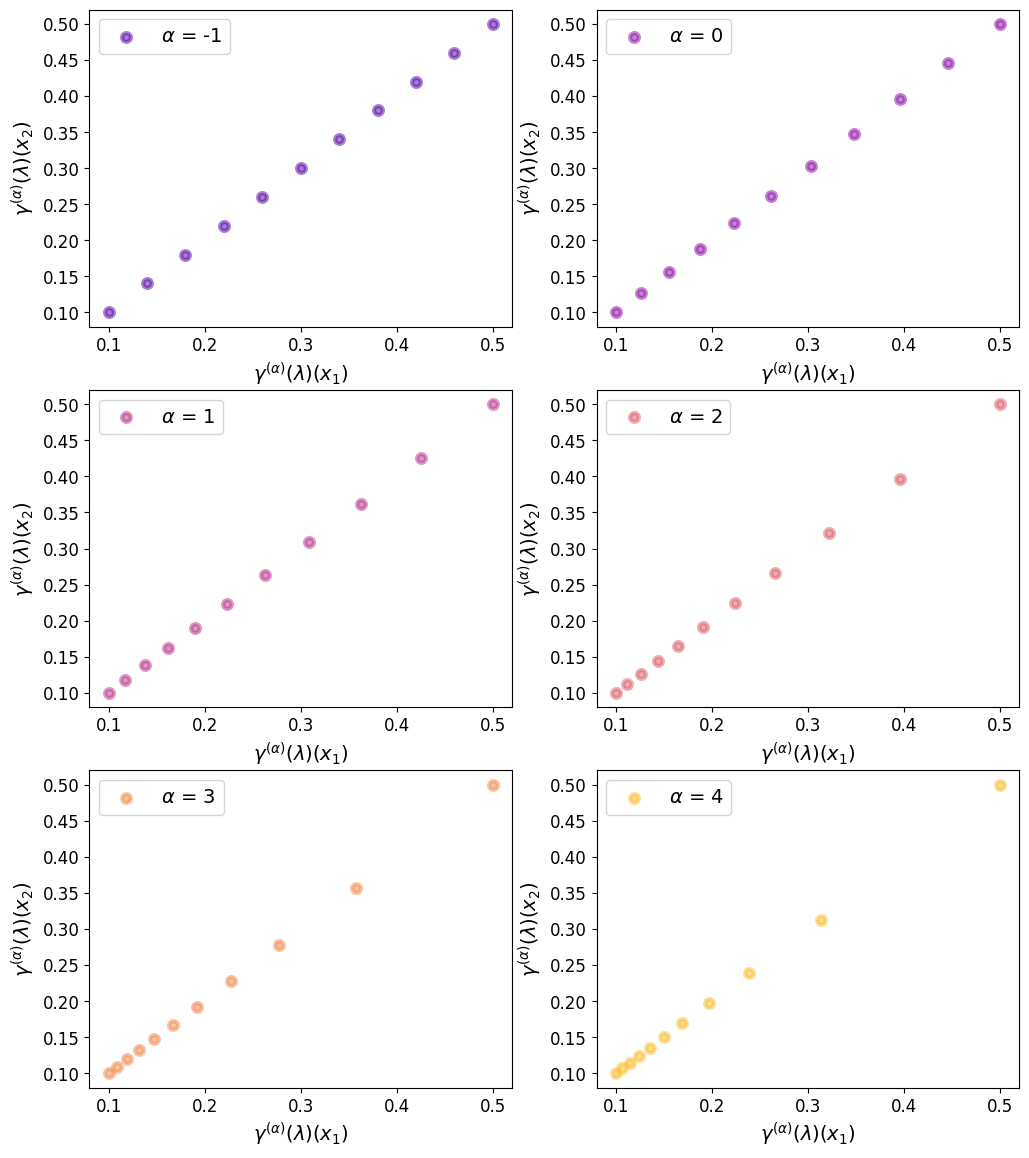}
    \caption{Examples of $\alpha$-geodesics connecting $p_s(\bm{x})$ and $p_t(\bm{x})$ in two dimensional case.
    Here, we assume that $p_s(\bm{x}) = (0.1, 0.1)$ and $p_t(\bm{x}) = (0.9, 0.9).$}
    \label{fig:alpha_geodesics_2d}
\end{figure}

\subsection{Generalized Geodesics and Information Geometry}
\label{apd:summary_of_geometry}
In our manuscript, we have tried to deliver to the reader the usefulness of the geometric interpretation of statistical procedures, while eliminating as much as possible the notations of differential geometry.
However, in order to make our manuscript more self-contained, in this section we introduce a few selected geometric concepts that may help in understanding our work.

Let $g_{ij}$ be a Riemannian metric, particularly the Fisher information matrix for the statistical manifold.
Then the most simple connection on the manifold is defined by the following Christoffel symbols of the first kind.
\begin{align}
    \Gamma_{ij, k} \coloneqq \frac{1}{2}\left(\partial_i g_{jk} + \partial_j g_{ik} - \partial_k g_{ij}\right).
\end{align}
The Levi-Civita connection $\nabla^{(0)}$ is defined as
\begin{align}
    g(\nabla^{(0)}_{\partial_i}\partial_j, \partial_k) = \Gamma_{ij,k}.
\end{align}
Here, the operation $\nabla\colon\mathcal{X}(\mathcal{M})\times\mathcal{X}(\mathcal{M})\to\mathcal{X}(\mathcal{M})$ on a differentiable manifold $\mathcal{M}$ is called a linear connection if it satisfies
\begin{itemize}
    \item[i)] $\nabla_XY$ is $\mathcal{F}(\mathcal{M})$-linear in $X$,
    \item[ii)] $\nabla_XY$ is $\mathbb{R}$-linear in $Y$, and
    \item[iii)] $\nabla$ satisfies the Leibniz rule, that is
    \begin{align}
        \nabla_X(fY) = (Xf)Y + f\nabla_XY, \quad \forall f \in \mathcal{F}(\mathcal{M}),
    \end{align}
\end{itemize}
where $\mathcal{X}(\mathcal{M})$ is the set of vector fields on $\mathcal{M}$ and $\mathcal{F}(\mathcal{M})$ is the set of all differentiable functions on $\mathcal{M}$.
It is worth nothing that the superscript of $\nabla^{(0)}$ denotes a parameter of the connection, and the generalized connection induced by the $\alpha$-divergence is denoted as $\nabla^{(\alpha)}$.
In this case, $\alpha=0$ corresponds to the Levi-Civita connection.
Two other important special cases of the $\alpha$-connection are the $\nabla^{(1)}$- and $\nabla^{(-1)}$-connections, given as follows.
\begin{align}
    g(\nabla^{(1)}_{\partial_i}\partial_j, \partial_k) &= \Gamma^{(1)}_{ij,k} \coloneqq \mathbb{E}\left[(\partial_i\partial_j\ell)(\partial_k\ell)\right], \\
    g(\nabla^{(-1)}_{\partial_i}\partial_j, \partial_k) &= \Gamma^{(-1)}_{ij, k} \coloneqq \mathbb{E}\left[(\partial_i\partial_j\ell + \partial_i\ell\partial_j\ell)(\partial_k\ell)\right].
\end{align}
Note that $\alpha$-divergence with $\alpha=\pm 1$ is the KL-divergence and its dual.
Here, the relationship between $\alpha$-divergence and $\alpha$-connection is that under $\alpha$-connection, the $\alpha$-geodesic given by Eq.~\eqref{eq:alpha_geodesics} becomes a straight path, and $\alpha$-divergence becomes its minimizer.
The $\alpha$-connection is given as
\begin{align}
    g(\nabla^{(\alpha)}_{\partial_i}\partial_j, \partial_k) = \Gamma^{(\alpha)}_{ij,k} \coloneqq \mathbb{E}\left[\left(\partial_i\partial_j\ell + \frac{1 - \alpha}{2}\partial_i\ell\partial_j\ell\right)\partial_k\ell\right],
\end{align}
and
\begin{align}
    \nabla^{(\alpha)} = \frac{1 + \alpha}{2}\nabla^{(1)} + \frac{1 - \alpha}{2}\nabla^{(-1)}.
\end{align}

For example, let us consider the following Gaussian distribution.
\begin{align}
    p(x; \bm{\theta}) = \frac{1}{\sqrt{2\pi\sigma^2}}\exp\left\{-\frac{(x - \mu)^2}{2\sigma^2}\right\},
\end{align}
where $\bm{\theta} = (\theta_1, \theta_2) = (\mu, \sigma)$.
The Fisher information matrix is given by
\begin{align}
    g_{ij} = \begin{pmatrix}
    \frac{1}{\sigma^2} & 0 \\
    0 & \frac{2}{\sigma^2}
    \end{pmatrix}.
\end{align}
Then, from the simple calculation, the Christoffel symbols of first and second kind are given as follows.
\begin{align}
    \Gamma_{11,2} &= \frac{1}{\sigma^3}, \\
    \Gamma_{12,1} &= -\frac{1}{\sigma^3}, \\
    \Gamma_{22,2} &= -\frac{2}{\sigma^3}, \\
    \Gamma^1_{ij} &= \begin{pmatrix}
        0 & -\frac{1}{\sigma} \\
        -\frac{1}{\sigma} & 0
    \end{pmatrix}, \\
    \Gamma^2_{ij} &= \begin{pmatrix}
        \frac{1}{2\sigma} & 0 \\
        0 & -\frac{1}{\sigma}
    \end{pmatrix},
\end{align}
and we can see that the geodesics are solutions of the following ordinary differential equations.
\begin{align}
    \ddot{\mu} - \frac{2}{\sigma}\dot{\mu}\dot{\sigma} &= 0, \\
    \ddot{\sigma} + \frac{1}{2\sigma}(\dot{\mu})^2 - \frac{1}{\sigma}(\dot{\sigma})^2 &= 0.
\end{align}
We can then have
\begin{align}
    \frac{\ddot{\mu}}{\dot{\mu}} = \frac{2\dot{\sigma}}{\sigma} \Leftrightarrow \dot{\mu} = c\sigma^2,
\end{align}
with some constant $c$.
In the case of $c \neq 0$ we have
\begin{align}
    \sigma(s)^2 + (\mu(s) - K)^2\frac{J}{\sigma^4} = c^2,
\end{align}
where $K > 0$ and $J > 0$ are positive constants, and the geodesics are half-ellipses, with $\sigma > 0$.
In the case of $c = 0$,
\begin{align}
    \sigma(s) = \sigma(0)e^{\sqrt{J}s},
\end{align}
and the geodesics are vertical half-lines.
For the manifold of Gaussian distributions, the Christoffel coefficients of first kind are written as
\begin{align}
    \Gamma^{(\alpha)}_{11,1} &= \Gamma^{(\alpha)}_{21,2} = \Gamma^{(\alpha)}_{12,2} = \Gamma^{(\alpha)}_{22,1} = 0, \\
    \Gamma^{(\alpha)}_{11,2} &= \frac{1 - \alpha}{\sigma^3}, \\
    \Gamma^{(\alpha)}_{12,1} &= \Gamma^{(\alpha)}_{21,1} = \frac{1 + \alpha}{\sigma^3}, \\
    \Gamma^{(\alpha)}_{22,2} &= \frac{2(1 + 2\alpha)}{\sigma^3},
\end{align}
and the Christoffel symbols of second kind are as follows.
\begin{align}
    \Gamma^{1 (\alpha)}_{ij} &= g_{11}\Gamma^{(\alpha)}_{ij, 1} + g_{12}\Gamma^{(\alpha)}_{ij,2} = \sigma^2\Gamma^{(\alpha)}_{ij,1} \nonumber \\
    &= \sigma^2\begin{pmatrix}
        0 & -\frac{1 + \alpha}{\sigma^3} \\
        -\frac{1 + \alpha}{\sigma^3} & 0
    \end{pmatrix} = \begin{pmatrix}
        0 & -\frac{1 + \alpha}{\sigma} \\
        -\frac{1 + \alpha}{\sigma} & 0
    \end{pmatrix}.
\end{align}
The ordinary differential equations for the $\alpha$-autoparallel curves are given as
\begin{align}
    \ddot{\mu} - \frac{2(1 + \alpha)}{\sigma}\dot{\sigma}\dot{\mu} &= 0, \\
    \ddot{\sigma} + \frac{1 - \alpha}{2\sigma}\dot{\mu}^2 - \frac{1 + 2\alpha}{\sigma}\dot{\sigma}^2 &= 0.
\end{align}
The first equation can be transformed as
\begin{align}
    \frac{\ddot{\mu}}{\mu} = 2(1 + \alpha)\frac{\dot{\sigma}}{\sigma} &\Leftrightarrow \frac{d}{ds}\ln\dot{\mu} = 2(1 + \alpha)\frac{d}{ds}\ln\sigma \nonumber \\
    &\Leftrightarrow \ln\dot{\mu} = 2(1 + \alpha)\ln\sigma + c_0 \nonumber \\
    &\Leftrightarrow \dot{\mu} = c\sigma^{2(1 + \alpha)},
\end{align}
with constant $c$.
Then, we have
\begin{align}
    \ddot{\sigma} + \frac{1 - \alpha}{2\sigma}c^2\sigma^{4(1 + \alpha)} - \frac{1 + 2\alpha}{\sigma}\dot{\sigma}^2 &= 0 \nonumber \\
    \sigma^{k+1}du + \left(\frac{1 - \alpha}{2}c^2\sigma^{4(\alpha + 1) + k} - (1 + 2\alpha)\sigma^k\mu^2\right)d\sigma &= 0,
\end{align}
where $u = \dot{\sigma}$.
Here, the following must be satisfied.
\begin{align}
    \frac{\partial}{\partial\sigma}\sigma^{k + 1}u = \frac{\partial}{\partial\mu}\left(\frac{1 - \alpha}{2}c^2\sigma^{4(\alpha + 1) + k} - (1 + 2\alpha)\sigma^k\mu^2\right).
\end{align}
By considering this condition, we can determine $k + 1 = -(4\alpha + 2)$, and we have
\begin{align}
    u\sigma^{-(4\alpha + 2)}du + \left(\frac{1 - \alpha}{2}c^2\sigma - (1 + 2\alpha)u^2\sigma^{-(4\alpha + 3)}\right)d\sigma = 0.
\end{align}
Hence, for a constant $E$, we can solve for $\sigma$ as follows.
\begin{align}
    \frac{\mu^2}{\sigma^2(2\alpha + 1)} + \frac{1 - \alpha}{2}c^2\sigma^2 = E &\Leftrightarrow \left(\frac{\dot{\sigma}}{\sigma^{2\alpha + 1}}\right)^2 + \frac{1 - \alpha}{2}c^2\sigma^2 = E \nonumber \\
    &\Leftrightarrow \left(\frac{\dot{\sigma}}{\sigma^{2\alpha + 1}}\right)^2 = E - \frac{1 - \alpha}{2}c^2\sigma^2 \nonumber \\
    &\Leftrightarrow \int\frac{d\sigma}{\sigma^{2\alpha + 1}\sqrt{E - \frac{1 - \alpha}{2}c^2\sigma^2}} = \pm s + s_0 \nonumber \\
    &\Leftrightarrow \int \frac{d\sigma}{\sigma^{2\alpha + 1}\sqrt{C^2 - \sigma^2}} = (\pm s + s_0)\sqrt{\frac{1 - \alpha}{2}}c,
\end{align}
where
\begin{align}
    C = C_\alpha = \frac{2E}{c}\frac{1}{1 - \alpha}.
\end{align}
Let $t = \sigma^2$ and $v = \sqrt{C^2 - t}$, we have
\begin{align}
    \int\frac{d\sigma}{\sigma^{2\alpha + 1}\sqrt{C^2 - \sigma^2}} &= \int\frac{dt}{2\sigma^{2(\alpha + 1)}\sqrt{C^2 - \sigma^2}} \nonumber \\
    &= \int \frac{dt}{2t^{(\alpha + 1)}\sqrt{C^2 - 1}} \nonumber \\
    &= \int\frac{-2vdv}{2t^{\alpha + 1}v} \nonumber \\
    &= -\int\frac{dv}{(C^2 - v^2)^{\alpha + 1}},
\end{align}
and then,
\begin{align}
    -\int\frac{dv}{(C^2 - v^2)^{\alpha + 1}} = (\pm s + s_0)\sqrt{\frac{1 - \alpha}{2}}c.
\end{align}
The $\mu$ is given by
\begin{align}
    \mu = c\int\sigma^{2(1 + \alpha)}(s)ds.
\end{align}
\paragraph{The Case of $\alpha = -1$}
In the case of $\alpha = -1$, we have
\begin{align}
    -v - K = (\pm s + s_0)\sqrt{\frac{1 - \alpha}{2}}c,
\end{align}
with solution
\begin{align}
    \sigma^2(s) = C^2 - \left((\pm s + s_0)\sqrt{\frac{1 - \alpha}{2}}c + K\right)^2,
\end{align}
for a constant $K$.
Then, we have
\begin{align}
    \mu(s) = cs + \mu(0).
\end{align}
\paragraph{The case of $\alpha=1/2$}
In the case of $\alpha = 1/2$, since
\begin{align}
    \int\frac{dv}{(C^2 - v^2)^{3/2}} = \frac{v}{C^2\sqrt{C^2 - v^2}},
\end{align}
and we have
\begin{align}
    -\frac{v}{C^2\sqrt{C^2 - v^2}} = (\pm s + s_0)\frac{c}{2} + K,
\end{align}
with solution
\begin{align}
    \sigma(s) = \frac{C}{\sqrt{1 + C^4\left(\frac{c}{2}(\pm s + S_0) + K\right)^2}},
\end{align}
for a constant $K$.
The, we have
\begin{align}
    \mu(s) = c\int \sigma^3(s) ds.
\end{align}

\subsection{Approximation of Arc Lengths of Generalized Geodesics for Equally Spaced Transitions}
\label{apd:approximation_arc_length}
Let $\gamma^{(\alpha)}(t)$ be the $\alpha$-geodesic, and we consider its Taylor expansion as
\begin{align}
    \gamma^{(\alpha)i}(t) = \gamma^{(\alpha)i}(0) + t\frac{d\gamma^{(\alpha)i}}{dt}\Bigr|_{t=0} + \frac{t^2}{2}\frac{d^2\gamma^{(\alpha)i}}{dt^2}\Bigr|_{t=0}.
\end{align}
Here, $\alpha$-geodesics satisfy the following Euler-Lagrange equation.
\begin{align}
    \frac{d^2\gamma^{(\alpha)i}}{dt^2} + \Gamma^{(\alpha)i}_{jk}\frac{d\gamma^{(\alpha)j}}{dt}\frac{d\gamma^{(\alpha)k}}{dt} = 0,
\end{align}
where
\begin{align}
    \Gamma^{(\alpha)}_{ij,k} = \mathbb{E}\left[\left(\partial_i\partial_j\ell + \frac{1-\alpha}{2}\partial_i\ell\partial_j\ell\right)\partial_k\ell\right],
\end{align}
and $\ell = \ln p(x)$.
Then, we can rewrite as
\begin{align}
    \gamma^{(\alpha)i}(t) = \gamma^{(\alpha)i}(0) + t\frac{d\gamma^{(\alpha)i}(t)}{dt}\Bigr|_{t=0} - \frac{t^2}{2}\left(\Gamma^{(\alpha)i}_{jk}\frac{d\gamma^{(\alpha)j}}{dt}\frac{d\gamma^{(\alpha)k}}{dt}\right)\Bigr|_{t=0}.
\end{align}
The length of this curve connecting $\gamma^{(\alpha)}(0) = p_s$ and $\gamma^{(\alpha)}(1) = p_t$ is defined as
\begin{align}
    \|\gamma^{(\alpha)}\| \coloneqq \int^1_0 \sqrt{\sum_{i,j}g_{ij}\frac{d\gamma^{(\alpha)i}(t)}{dt}\frac{d\gamma^{(\alpha)j}(t)}{dt}}.
\end{align}
Finally, we can see that equally spaced transitions $\lambda_1,\dots,\lambda_m$ on the $\alpha$-geodesic satisfies
\begin{align}
    \int^{\lambda_i}_{\lambda_{i-1}} \sqrt{\sum_{i,j}g_{ij}\frac{d\gamma^{(\alpha)i}(t)}{dt}\frac{d\gamma^{(\alpha)j}(t)}{dt}} = \int^{\lambda_j}_{\lambda_{j-1}} \sqrt{\sum_{i,j}g_{ij}\frac{d\gamma^{(\alpha)i}(t)}{dt}\frac{d\gamma^{(\alpha)j}(t)}{dt}}, \quad 1 \leq i, j \leq m.
\end{align}

\begin{figure}
    \centering
    \includegraphics[width=\linewidth]{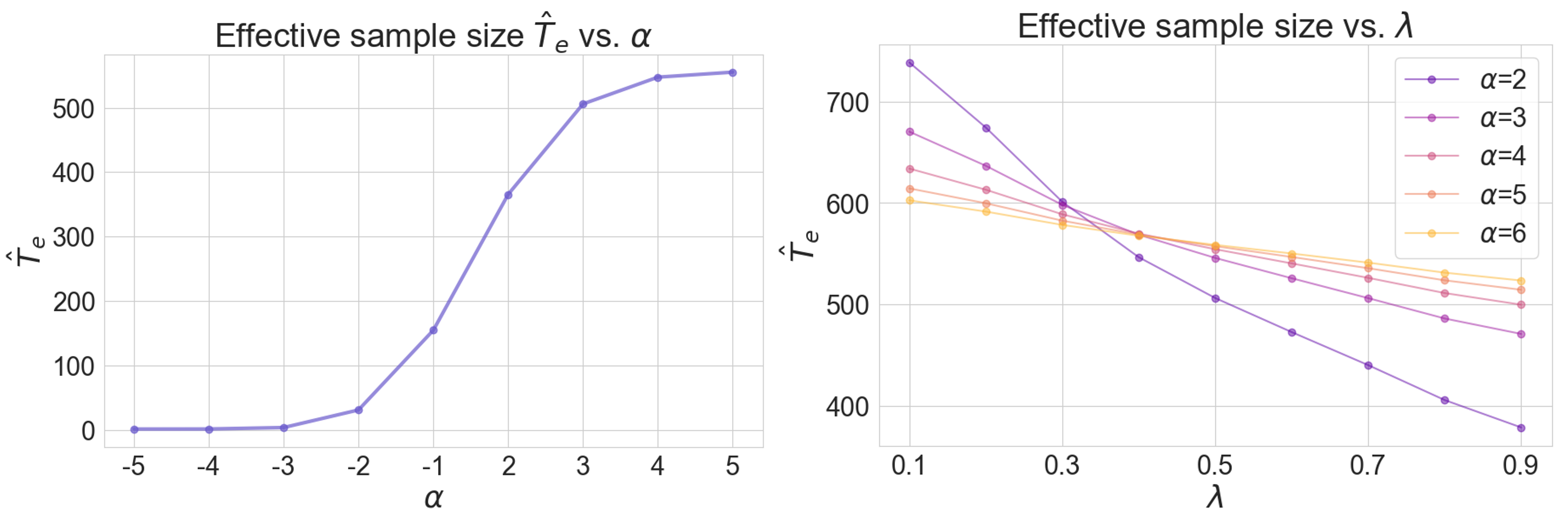}
    \caption{Relationship between ESS and the parameters $\alpha$ and $\lambda$.}
    \label{fig:effective_sample_size}
\end{figure}

\subsection{Additional Discussion on Variance and Effective Sample Size}
We provide some insight into the question about the choice of $\alpha$ through the following analysis.
\begin{proposition}
    \label{prop:variance_gimdre}
    For the estimator $\hat{I}$, we can see that the variance $\mathrm{Var}(\hat{I})$ decreases as $\alpha \to +\infty$.
   %\begin{itemize}
   %    \item The variance $\mathrm{Var}(\hat{I})$ decreases as $\alpha \to +\infty$ and $\lambda \to 0$, and
   %    \item $\mathrm{Var}(\hat{I})$ increases as $\alpha \to -\infty$ and $\lambda \to 1$.
   %\end{itemize} 
\end{proposition}
\begin{proof}
    The variance of $\hat{I}$ with $p_\alpha^\lambda$ is calculated as follows.
    \begin{align}
        \mathrm{Var}_{\alpha,\lambda}(\hat{I}) &= \int\left(\frac{\gamma^{(\alpha)}(\lambda)(x)}{p_s(x)}g(x) - I\right)^2 p_s(x)dx \nonumber \\
        &= \int\left(\frac{\gamma^{(\alpha)}(\lambda)(x)^2}{p_s(x)^2}g(x)^2 - 2\frac{\gamma(\alpha)(\lambda)(x)}{p_s(x)}g(x)I + I^2\right) p_s(x) dx \nonumber \\
        &= \int \frac{\gamma^{(\alpha)}(\lambda)(x)^2}{p_s(x)}g(x)^2dx - I^2 \nonumber \\
        &= \int\frac{\gamma^{(\alpha)}(\lambda)(x)}{p_s(x)}g(x)^2 \gamma^{(\alpha)}(\lambda)(x)dx - I^2 \nonumber \\
        &= \mathbb{E}_{\gamma^{(\alpha)}(\lambda)}\left[\frac{\gamma^{(\alpha)}(\lambda)(x)}{p_s(x)}g(x)^2\right] - I^2 \nonumber \\
    \end{align}
    Thus, the variance of the estimator by sampling from the alpha-mixture depends on the expectation of $g(x)$ scaled by the weight $w(x) = \gamma^{(\alpha)}(\lambda)(x) / p_s(x)$.
    Here, for all $\alpha$, we have $\frac{\partial}{\partial \alpha}\frac{\gamma^{(\alpha)}(\lambda)(x)}{p_s(x)} \leq 0$, and then we have the proof. 
\end{proof}
The above proof is based on the monotonicity of $\alpha$-geodesics with respect to $\alpha$.
Therefore, we give a few more observations on this important key monotonicity.
Figure~\ref{fig:gamma_monotonically} shows the monotonicity of the $\alpha$-geodesic.
We use $p_s=0.1$ and $p_t=0.9$.
The left panel shows that the value of the $\alpha$-geodesic decreases monotonically as $\alpha$ is increased. Here, $\lambda$ is set to a constant of $0.5$.
Recall that $\alpha=-1$ is an ordinary weighted average, and we can see that the geodesic value is $(p_s(\bm{x}) + p_t(\bm{x})) / 2$ at $\alpha=-1$ in the figure.
The right panel shows that the $\alpha$-geodesic approaches $p_s$ at $\lambda\to 0$ and $p_t$ at $\lambda\to 1$ for all $\alpha$.
Further, for larger values of $\alpha$, the speed at which the $\alpha$-geodesic approaches $p_t$ is found to be non-uniform.
This behavior, in which the geodesic gradually approaches the target distribution from the source distribution, can be expected to affect the stability of the estimator.
In addition, Figure~\ref{fig:alpha_geodesics_2d} shows the examples of $\alpha$-geodesics connecting $p_s(\bm{x})$ and $p_t(\bm{x})$ in two dimensional case.
Here, we assume that $p_s(\bm{x}) = (0.1, 0.1)$ and $p_t(\bm{x}) = (0.9, 0.9).$
In this figure, $\lambda \in \{0, 0.1, 0.2, 0.3, 0.4, 0.5, 0.6, 0.7, 0.8, 0.9, 1.0\}$ for all $\alpha$.
It can be seen that for $\alpha = -1$, one approaches the target distribution at a uniform rate as the $\lambda$ increases, while for larger values of $\alpha$, there is an acceleration in approaching the target distribution.
The above results follow from the monotonicity of $\alpha$-geodesics with respect to $\alpha$.
Note that, for arbitrary $\lambda \in [0, 1]$,
% \begin{align}
    % $\lim_{\alpha\to+\infty}w(\bm{x}) = \min\{p_s(\bm{x}), p_t(\bm{x})\} / p_s(\bm{x}) = \min\left\{1, p_t(\bm{x}) / p_s(\bm{x}) \right\}$.
    $\lim_{\alpha\to+\infty}w(\bm{x}) = \min\left\{1, p_t(\bm{x}) / p_s(\bm{x}) \right\}$.
% \end{align}
This can be seen as a truncation of the density ratio by $1$ in the case of $\alpha\to +\infty$.

% Proposition: Effective Sample Size
% alpha -> +inftyで改善する！
Finally, the choice of the geodesic parameter $\alpha$ is discussed in more detail.
The appropriate parameter $\alpha$ for sampling from $\alpha$-geodesics depends on $g(\bm{x})$, the target of the Monte Carlo integration.
Let $\bar{Z}_{\alpha} \coloneqq \sum^T_{i=1} w(\bm{x}^*)\bm{x}_i / \sum^T_{j=1}w(\bm{x}_j)$, this variance is given by
\begin{align}
    \mathrm{Var}(\bar{Z}_{\alpha}) = \mathrm{Var}(\bm{x})\frac{\sum^T_{i=1}w(\bm{x})^2}{\left(\sum^T_{i=1}w(\bm{x})\right)^2} = \mathrm{Var}(\bm{x})\frac{\sum^T_{i=1}\left\{1 - \lambda + \lambda\left(1 / \hat{r}(\bm{x}^s_i)\right)^{\frac{1-\alpha}{2}}\right\}^{\frac{4}{1-\alpha}}}{\left(\sum^T_{i=1}\left\{1 - \lambda + \lambda\left(1 / \hat{r}(\bm{x}^s_i)\right)^{\frac{1-\alpha}{2}}\right\}^{\frac{2}{1-\alpha}}\right)^2}.
\end{align}
On the other hand, the variance of $\bar{Z} = \frac{1}{T_e}\sum^{T_e}_{i=1}\bm{x}_i$ is $\mathrm{Var}(\bar{Z}) = \mathrm{Var}(\bm{x}) / T_e$ for some $T_e > 0$.
Then, let $\mathrm{Var}(\bar{Z}_\alpha) = \mathrm{Var}(\bar{Z})$ and solve for $T_e$ to obtain the following.
\begin{align}
    \hat{T}_e = \frac{\left(\sum^T_{i=1}w(\bm{x})\right)^2}{\sum^T_{i=1}w(\bm{x})^2} = \frac{\left(\sum^T_{i=1}\left\{1 - \lambda + \lambda\left(1 / \hat{r}(\bm{x}^s_i)\right)^{\frac{1-\alpha}{2}}\right\}^{\frac{2}{1-\alpha}}\right)^2}{\sum^T_{i=1}\left\{1 - \lambda + \lambda\left(1 / \hat{r}(\bm{x}^s_i)\right)^{\frac{1-\alpha}{2}}\right\}^{\frac{4}{1-\alpha}}}.
\end{align}
$\hat{T}_e$ represents the number of data required to obtain the same accuracy as estimation by $\bar{Z}_\alpha$ when estimating the expected value of a sample by $\bar{Z}$.
It is called the effective sample size (ESS).
ESS holds for $\hat{T}_{e} \leq T$, and the equality holds when w is a constant.
When the value of $\hat{T}_e$ is small, estimation by $\bar{Z}_\alpha$ requires more samples than $\bar{Z}$, meaning that estimation is inefficient.
From the monotonicity of the $\alpha$-geodesic with respect to $\alpha$, we obtain the following.
\begin{proposition}
    \label{prop:ess}
    The estimation efficiency in terms of ESS of GIMDRE improves with $\alpha\to +\infty$.
\end{proposition}
Recall that ESS $\hat{T}_e$ is given as
\begin{align}
    \hat{T}_e = \frac{\left(\sum^T_{i=1}w(\bm{x})\right)^2}{\sum^T_{i=1}w(\bm{x})^2} = \frac{\left(\sum^T_{i=1}\left\{1 - \lambda + \lambda\left(1 / \hat{r}(\bm{x}^s_i)\right)^{\frac{1-\alpha}{2}}\right\}^{\frac{2}{1-\alpha}}\right)^2}{\sum^T_{i=1}\left\{1 - \lambda + \lambda\left(1 / \hat{r}(\bm{x}^s_i)\right)^{\frac{1-\alpha}{2}}\right\}^{\frac{4}{1-\alpha}}}.
\end{align}
Here, we can rewrite this as follows.
\begin{align}
    \hat{T}_e &= \frac{T}{1 + \frac{\frac{1}{T}\sum^T_{i=1}\left\{1 - \lambda + \lambda\left(1 / \hat{r}(\bm{x}^s_i)\right)^{\frac{1-\alpha}{2}}\right\}^{\frac{4}{1-\alpha}} - \left(\frac{1}{T}\sum^T_{i=1}\left\{1 - \lambda + \lambda\left(1 / \hat{r}(\bm{x}^s_i)\right)^{\frac{1-\alpha}{2}}\right\}^{\frac{2}{1-\alpha}}\right)^2}{\left(\frac{1}{T}\sum^T_{i=1}\left\{1 - \lambda + \lambda\left(1 / \hat{r}(\bm{x}^s_i)\right)^{\frac{1-\alpha}{2}}\right\}^{\frac{2}{1-\alpha}}\right)^2}} \nonumber \\
    &= \frac{T}{1 + V^2},
\end{align}
where
\begin{align}
    V \coloneqq \frac{\frac{1}{T}\sum^T_{i=1}\left\{1 - \lambda + \lambda\left(1 / \hat{r}(\bm{x}^s_i)\right)^{\frac{1-\alpha}{2}}\right\}^{\frac{4}{1-\alpha}} - \left(\frac{1}{T}\sum^T_{i=1}\left\{1 - \lambda + \lambda\left(1 / \hat{r}(\bm{x}^s_i)\right)^{\frac{1-\alpha}{2}}\right\}^{\frac{2}{1-\alpha}}\right)^2}{\left(\frac{1}{T}\sum^T_{i=1}\left\{1 - \lambda + \lambda\left(1 / \hat{r}(\bm{x}^s_i)\right)^{\frac{1-\alpha}{2}}\right\}^{\frac{2}{1-\alpha}}\right)^2}.
\end{align}
The value $V$ is called the coefficient of variation, and from the fact that $V^2 \geq 0$, we can see that $\hat{T}_e \leq T$.
Figure~\ref{fig:effective_sample_size} shows the relationship between ESS and the parameters $\alpha$ and $\lambda$.
In this figure, $p_s=\mathcal{N}(8, 3)$ and $p_t=\mathcal{N}(0, 2)$.
The left panel of Figure~\ref{fig:effective_sample_size} shows the relationship between $\alpha$ and ESS with $\lambda=0.5$ as a constant.
We see that the estimation efficiency in terms of the ESS is improved by using a large $\alpha$.
Furthermore, the right panel of Figure~\ref{fig:effective_sample_size} shows that the estimation efficiency deteriorates as one moves away from the source distribution with larger $\lambda$.
However, it can be observed that this degradation can be reduced by using a large $\alpha$.
Figures~\ref{fig:ess_lognormal} and~\ref{fig:ess_powerlaw} show the results for the two skewed distributions, Log-normal and Power-law, and reveal similar observations.
Here, the density functions of these two distributions $p_{\mathrm{LogNormal}}$ and $p_{\mathrm{PowerLaw}}$ are defined as follows.
\begin{align}
    p_{\mathrm{LogNormal}}(x; \mu, \sigma) &= \frac{1}{x\sigma\sqrt{2\pi}}\exp\left\{-\frac{(\ln x - \mu)^2}{2\sigma^2}\right\}, \quad x \in (0, +\infty), \ \mu \in (-\infty, +\infty),\ \sigma > 0, \nonumber \\
    p_{\mathrm{PowerLaw}}(x; a) &= ax^{a - 1}, \quad 0 \leq 1,\ a > 0. \nonumber
\end{align}
For Log-normal distributions $p_s = \mathrm{LogNormal}(\mu_s, \sigma_s)$ and $\mathrm{LogNormal}(p_t, \sigma_t)$, we use $\mu_s \in \{3, 4\}, \sigma_s = 0.5$, $\mu_t =0$ and $\sigma_t \in \{1.5, 2, 2.5\}$.
Also, for Power-law distributions $p_s = \mathrm{PowerLaw}(a_s)$ and $p_t = \mathrm{PowerLaw}(a_t)$, we use $a_s = 3$ and $a_t \in \{0.05, 0.1, 0.15, 0.2, 0.25, 0.3\}$.

% p_s, p_t -> log-normal, Gamma, etc
\begin{figure}[t]
    \centering
    \includegraphics[width=\linewidth]{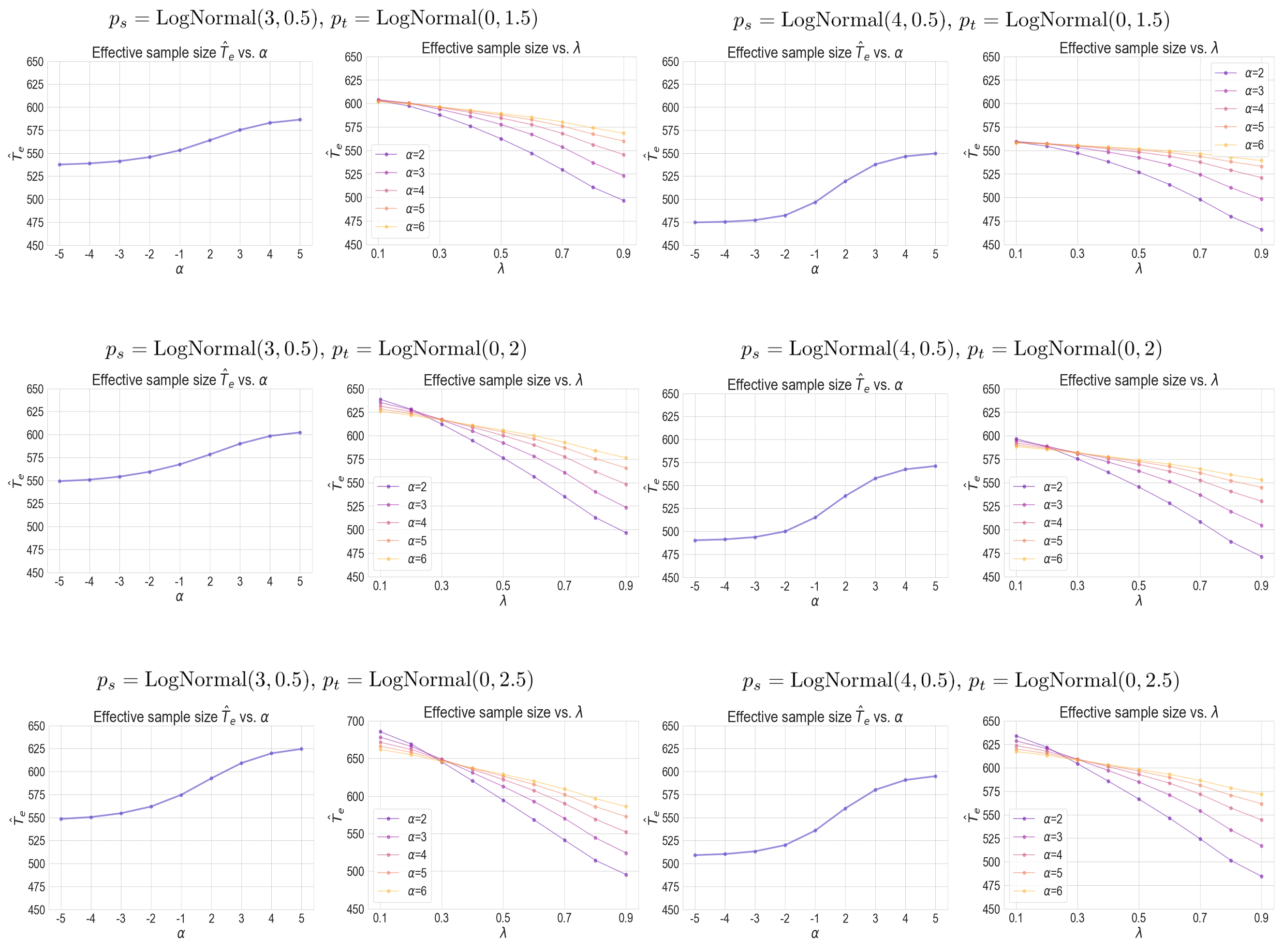}
    \caption{Relationship between ESS and the parameters $\alpha$ and $\lambda$ for Log-normal distributions.}
    \label{fig:ess_lognormal}
\end{figure}

\begin{figure}[t]
    \centering
    \includegraphics[width=\linewidth]{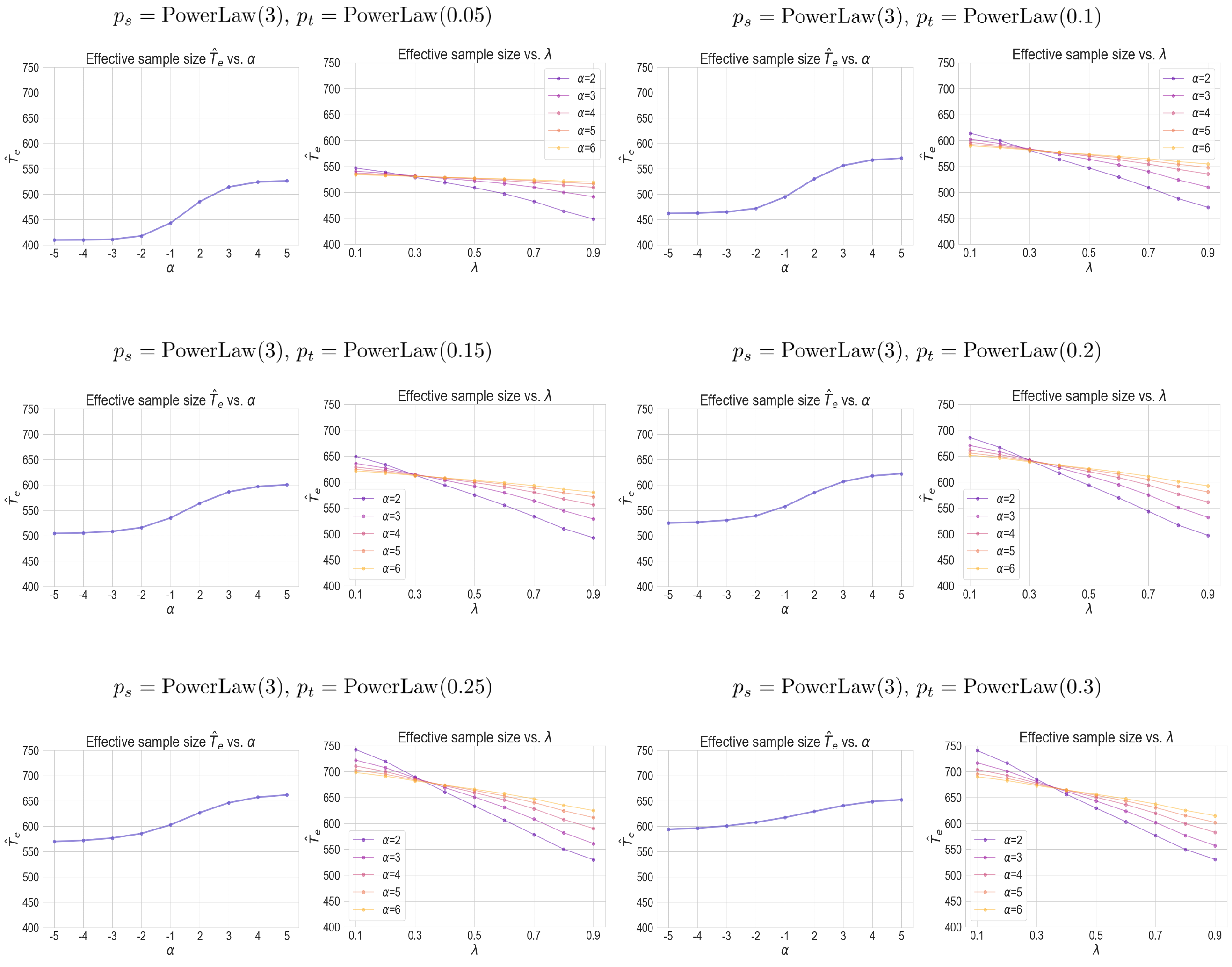}
    \caption{Relationship between ESS and the parameters $\alpha$ and $\lambda$ for Power-law distributions.}
    \label{fig:ess_powerlaw}
\end{figure}

\clearpage

\section{Details of Numerical Experiments and Additional Experimental Results}
\label{apd:experiments}
This appendix provides details of the numerical experiments and additional experimental results.

\subsection{Implementation Details}
\subsubsection{Software}
The results of all numerical experiments are obtained by our Python 3.11.0~\footnote{\url{https://www.python.org/downloads/release/python-3110/}} implementation.
We use numpy~\citep{harris2020array} for random variable generation and matrix calculations, scipy~\citep{virtanen2020scipy} for computing the density function of probability distributions, matplotlib~\footnote{\url{https://matplotlib.org/}} and seaborn~\footnote{\url{https://seaborn.pydata.org/}} for visualization of the results.
In addition, scikit-learn~\citep{pedregosa2011scikit} is used to implement logistic regression, and default parameters are used unless otherwise noted.

\subsubsection{Computing Resources}
All our numerical experiments are performed on a machine with 16 GiB of system memory and 4 vCPUs

\subsubsection{Base Model used for Density Ratio Estimation}
In our experiments, we use a kernel logistic regression classifier as the density ratio estimator.
A kernel logistic regression classifier employs a parametric model of the following form for expressing the class-posterior probability $p(y\mid\bm{x})$.
\begin{align}
    p(y\mid \bm{x}; \bm{\theta}) = \frac{1}{1 + \exp(-y \psi(\bm{x})^\top\bm{\theta})},
\end{align}
where $\psi(\bm{x}) \colon \mathbb{R}^d \to \mathbb{R}^c$ is a basis function vector and $\bm{\theta} \in \mathbb{R}^c$ is a parameter vector.
In our experiments, we use the following basis functions corresponding to the polynomial kernel and B-spline kernel in addition to the linear kernel.
The penalized log-likelihood maximization problem reduces the following minimization.
\begin{align}
    \hat{\bm{\theta}} \coloneqq \argmin_{\bm{\theta}\in\mathbb{R}^c}\left\{\sum^n_{i=1}\ln \left(1 + \exp\left(-y_i \psi(\bm{x}_i)^\top\bm{\theta}\right)\right) + \lambda\bm{\theta}^\top\bm{\theta}\right\},
\end{align}
where $\lambda\bm{\theta}^\top\bm{\theta}$ is a penalty term included for regularization process.
Thus, the loss function $L(y_i, \bm{\theta})$ is
\begin{align}
    L(y_i, \bm{\theta}) \coloneqq \ln\left(1 + \exp\left(-y_i\psi(\bm{x}_i)^\top\bm{\theta}\right)\right).
\end{align}
From Bayes' theorem, the density ratio can be expressed as
\begin{align}
    r(\bm{x}) &= \frac{p_s(\bm{x})}{p_t(\bm{x})} \nonumber \\
    &= \left(\frac{p(y = +1 \mid \bm{x})p(\bm{x})}{p(y = +1)}\right)\left(\frac{p(y = -1 \mid \bm{x})p(\bm{x})}{p(y = -1)}\right)^{-1} \nonumber \\
    &= \frac{p(y = -1)}{p(y = +1)}\frac{p(y = +1 \mid \bm{x})}{p(y = -1 \mid \bm{x})}.
\end{align}
The ratio $p(y = -1) / p(y = +1)$ can be approximated by
\begin{align}
    \frac{p(y = -1)}{p(y = +1)} \approx \frac{n_t / (n_t + n_s)}{n_s / (n_t + n_s)} = \frac{n_t}{n_s}.
\end{align}
A density ratio estimator $\hat{r}$ is then given by
\begin{align}
    \hat{r}(\bm{x}) &= \frac{n_t}{n_s}\frac{1 + \exp\left(\psi(\bm{x})^\top\hat{\bm{\theta}}\right)}{1 + \exp\left(-\psi(\bm{x})^\top\hat{\bm{\theta}}\right)} \nonumber \\
    &= \frac{n_t}{n_s}\frac{\exp\left(\psi(\bm{x})^\top\hat{\bm{\theta}}\right)\left\{\exp\left(-\psi(\bm{x})^\top\hat{\bm{\theta}}\right) + 1\right\}}{1 + \exp\left(-\psi(\bm{x})^\top\hat{\bm{\theta}}\right)} \nonumber \\
    &= \frac{n_t}{n_s}\exp\left(\psi(\bm{x})^\top\hat{\bm{\theta}}\right). \label{eq:logistic_regression_ratio}
\end{align}

\subsubsection{Estimator of Pearson Divergence}
The Pearson divergence $D_{\mathrm{PE}}[p \| q]$ is defined as follows.
\begin{align}
    D_{\mathrm{PE}}[p \| q] \coloneqq \frac{1}{2}\int_{\mathcal{X}}\left(\frac{p(\bm{x})}{q(\bm{x})} - 1\right)^2 q(\bm{x})d\bm{x}.
\end{align}
Here, we have
\begin{align}
    D_{\mathrm{PE}}[p \| q] &= \frac{1}{2}\int_{\mathcal{X}}\left(\frac{p(\bm{x})}{q(\bm{x})} - 1\right)^2 q(\bm{x})d\bm{x} \nonumber \\
    &= \frac{1}{2}\int_{\mathcal{X}}\left\{\left(\frac{p(\bm{x})}{q(\bm{x})}\right)^2 - 2\frac{p(\bm{x})}{q(\bm{x})} + 1\right\}q(\bm{x})d\bm{x} \nonumber \\
    &= \frac{1}{2}\int_{\mathcal{X}}\frac{p(\bm{x})^2}{q(\bm{x})} - 2p(\bm{x}) + q(\bm{x}) d\bm{x} \nonumber \\
    &= \frac{1}{2}\int_{\mathcal{X}}\frac{p(x)^2}{q(\bm{x})}d\bm{x} - \int_{\mathcal{X}}p(\bm{x})d\bm{x} + \frac{1}{2}\int_{\mathcal{X}}q(\bm{x})d\bm{x} \nonumber \\
    &= \frac{1}{2}\int_{\mathcal{X}}\frac{p(\bm{x})}{q(\bm{x})}p(\bm{x})d\bm{x} - \int_{\mathcal{X}}\frac{p(\bm{x})}{q(\bm{x})}q(\bm{x})d\bm{x} + \frac{1}{2} \nonumber \\
    &= \frac{1}{2}\int_{\mathcal{X}}r(\bm{x})p(\bm{x})d\bm{x} - \int_{\mathcal{X}}r(\bm{x})q(\bm{x})d\bm{x} + \frac{1}{2} = \frac{1}{2}\mathbb{E}_{p(\bm{x})}\left[r(\bm{x})\right] - \mathbb{E}_{q(\bm{x})}\left[r(\bm{x})\right] + \frac{1}{2}. \nonumber
\end{align}
The estimator of the Pearson divergence $\hat{D}_{\mathrm{PE}}[X^s \| X^t]$ is 
\begin{align}
    \hat{D}_{\mathrm{PE}}\left[X^s \| X^t\right] \coloneqq \frac{1}{2n_s}\sum^{n_s}_{i=1}\hat{r}(\bm{x}^s_i) - \frac{1}{n_t}\sum^{n_t}_{i=1}\hat{r}(\bm{x}^t_i) + \frac{1}{2}.
\end{align}
Combining with the estimator by Eq.~\eqref{eq:logistic_regression_ratio},
\begin{align}
    \hat{D}_{\mathrm{PE}}\left[X^s \| X^t\right] &= \frac{1}{2n_s}\sum^{n_s}_{i=1}\frac{n_t}{n_s}\frac{1 + \exp\left(\psi(\bm{x}^s_i)^\top\hat{\bm{\theta}}\right)}{1 + \exp\left(\-\psi(\bm{x}^s_i)^\top\hat{\bm{\theta}}\right)} - \frac{1}{n_t}\sum^{n_t}_{i=1}\frac{n_t}{n_s}\frac{1 + \exp\left(\psi(\bm{x}^t_i)^\top\hat{\bm{\theta}}\right)}{1 + \exp\left(\-\psi(\bm{x}^t_i)^\top\hat{\bm{\theta}}\right)} + \frac{1}{2} \nonumber \\
    &= \frac{n_t}{2n^2_s}\sum^{n_s}_{i=1}\frac{1 + \exp\left(\psi(\bm{x}^s_i)^\top\hat{\bm{\theta}}\right)}{1 + \exp\left(\-\psi(\bm{x}^s_i)^\top\hat{\bm{\theta}}\right)} - \frac{1}{n_s}\sum^{n_t}_{i=1}\frac{1 + \exp\left(\psi(\bm{x}^t_i)^\top\hat{\bm{\theta}}\right)}{1 + \exp\left(\-\psi(\bm{x}^t_i)^\top\hat{\bm{\theta}}\right)} + \frac{1}{2}.
\end{align}

\subsubsection{Analytic Calculation of Divergence}
In our numerical experiments, we compute KL-divergence between two Gaussian distributions $p(x) = \mathcal{N}(\mu_1, \sigma_1)$ and $q(x) = \mathcal{N}(\mu_2, \sigma_2)$ analytically.
The calculation is performed as follows.
\begin{align}
    D_{\mathrm{KL}}[p \| q] &= \int_\mathcal{X}\left\{\ln p(x) - \ln q(x)\right\}p(x)dx \nonumber \\
    &= \int_{\mathcal{X}}\left\{- \ln \sigma_1 - \frac{1}{2}\left(\frac{x - \mu_1}{\sigma_1}\right)^2 + \ln \sigma_2 + \frac{1}{2}\left(\frac{x - \mu_s}{\sigma^2}\right)^2\right\}\frac{1}{\sqrt{2\pi\sigma}}e^{-\frac{1}{2}\left(\frac{x - \mu_1}{\sigma_1}\right)^2}\nonumber \\
    &= \int_\mathcal{X}\left\{\ln\frac{\sigma_2}{\sigma_1} + \frac{1}{2}\left(\left(\frac{x - \mu_2}{\sigma_2}\right)^2 - \left(\frac{x - \mu_1}{\sigma_1}\right)^2\right)\right\}\frac{1}{\sqrt{2\pi\sigma}}e^{-\frac{1}{2}\left(\frac{x - \mu_1}{\sigma_1}\right)^2} dx \nonumber \\
    &= \mathbb{E}_{p(x)}\left[\ln\frac{\sigma_2}{\sigma_1} + \frac{1}{2}\left(\left(\frac{x - \mu_2}{\sigma_2}\right)^2 - \left(\frac{x - \mu_1}{\sigma_1}\right)^2\right)\right] \nonumber \\
    &= \ln\frac{\sigma_2}{\sigma_1} + \frac{1}{2\sigma_2^2}\mathbb{E}_{p(x)}\left[(x - \mu_2)^2\right] - \frac{1}{2\sigma_1^2}\mathbb{E}_{p(x)}\left[(x - \mu_1)^2\right] \nonumber \\
    &= \ln\frac{\sigma_2}{\sigma_1} + \frac{1}{2\sigma_2^2}\mathbb{E}_{p(x)}\left[(x - \mu_2)^2\right] - \frac{1}{2} \nonumber \\
    &= \ln\frac{\sigma_2}{\sigma_1} + \frac{1}{2\sigma_2^2}\left\{\mathbb{E}_{p(x)}\left[(x - \mu_1)^2\right] + 2(\mu_1 - \mu_2)\mathbb{E}\left[(x - \mu_1)\right] + (\mu_1 - \mu_2)^2\right\} - \frac{1}{2} \nonumber \\
    &= \ln\frac{\sigma_2}{\sigma_1} + \frac{\sigma_1^2 + (\mu_1 - \mu_2)^2}{2\sigma_2^2} - \frac{1}{2}.
\end{align}

\begin{table}[t]
    \caption{Evaluation results with different $\alpha$ and sample dimensions $d$ with polynomial kernel for Gaussian distributions.
    DRE refers the density ratio estimation without incremental mixtures.}
    \resizebox{\columnwidth}{!}{
    \centering
    \begin{tabular}{lccccc}
        \toprule
                     & $d=2$ & $d=3$ & $d=4$ & $d=5$ & \\
        \midrule
        $DRE$ & $35.98 (\pm 27.33)$ & $1335.85 (\pm 1447.51)$ & $5038.17 (\pm 2666.30)$ & $48586.57 (\pm 19828.46)$ \\
        \hline
         $\alpha=-1$ & $30.61 (\pm 31.03)$ & $715.52 (\pm 143)$ &  $723.75 (\pm 320.86)$ &  $1194.64 (\pm 400.47)$& \\
         $\alpha=3$  & $8.27 (\pm 9.55)$ & $63.14 (\pm 80.29)$ & $296.16 (\pm 339.05)$ & $480.41 (\pm 89.07)$ & \\
         $\alpha=7$  & $6.09 (\pm 0.02)$ & $19.81 (\pm 29.53)$ & $42.52 (\pm 126.12)$ &  $412.99 (\pm 71)$ & \\
         \bottomrule
    \end{tabular}
    }
    \label{tab:dimension_alpha_poly}
\end{table}

\begin{table}[t]
    \caption{Evaluation results with different $\alpha$ and sample dimensions $d$ with spline kernel for Gaussian distributions.
    DRE refers the density ratio estimation without incremental mixtures.}
    \resizebox{\columnwidth}{!}{
    \centering
    \begin{tabular}{lccccc}
        \toprule
                     & $d=2$ & $d=3$ & $d=4$ & $d=5$ & \\
        \midrule
        $DRE$ & $38.13 (\pm 4.04)$  & $485.40 (\pm 71.28)$ & $5167.16 (\pm 1710.77)$ & $48656.67 (\pm 14837.67)$\\
        \hline
         $\alpha=-1$ & $6.09 (\pm 0.01)$ & $123.38 (\pm 122)$& $489.74 (\pm 409.68)$ & $2767.53 (\pm 3596.53)$\\
         $\alpha=3$  &$6.10 (\pm 0.00)$ & $7.69 (\pm 0.01)$ & $22.59 (\pm 0.01)$ & $119.83 (\pm 0.02)$\\
         $\alpha=7$  & $6.10 (\pm 0.00)$ & $7.70 (\pm 0.01)$ & $22.60 (\pm 0.01)$ & $119.85 (\pm 0.01)$\\
         \bottomrule
    \end{tabular}
    }
    \label{tab:dimension_alpha_spline}
\end{table}

\begin{figure}[t]
    \centering
    \includegraphics[width=\linewidth]{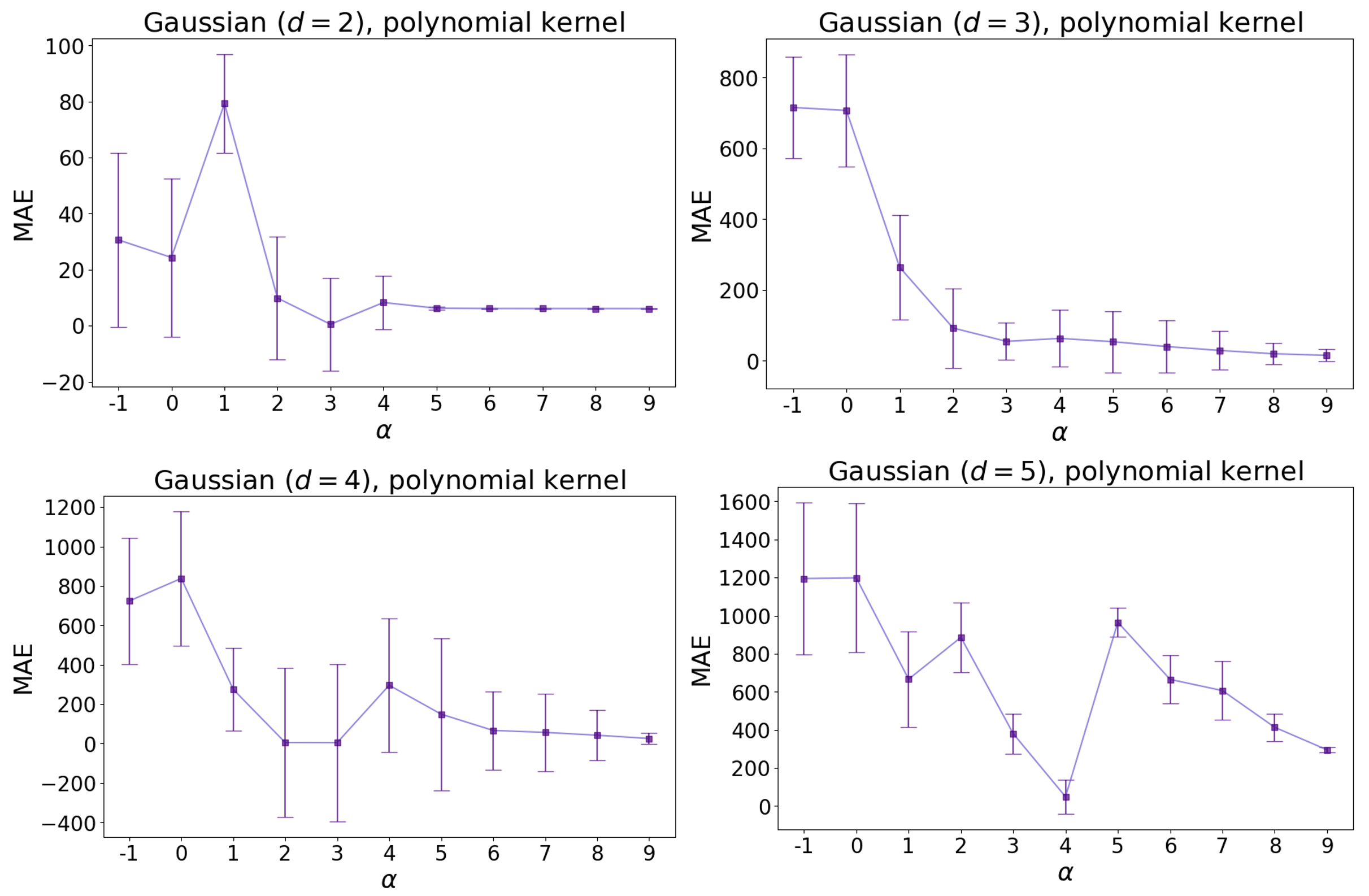}
    \caption{Evaluation results with different $\alpha$ and sample dimensions $d$ with polynomial kernel for Gaussian distributions.}
    \label{fig:poly_gaussian}
\end{figure}

\begin{figure}[t]
    \centering
    \includegraphics[width=\linewidth]{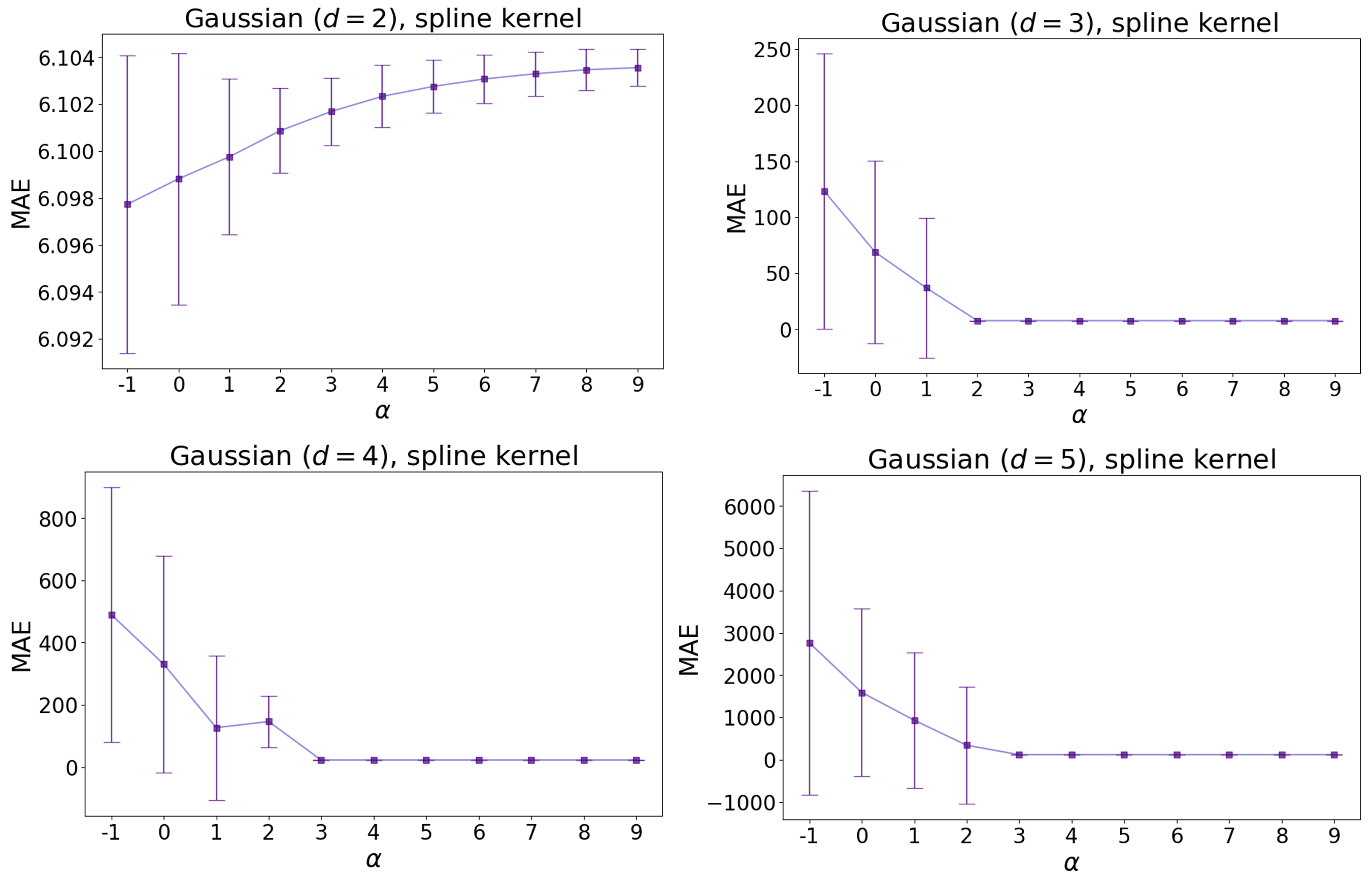}
    \caption{Evaluation results with different $\alpha$ and sample dimensions $d$ with spline kernel for Gaussian distributions.}
    \label{fig:spline_gaussian}
\end{figure}

\subsection{Additional Experimental Results}
% include original results of DRE
% Directly sampling from linear mixture
% Logistic Regression with kernel, for non-linear case
% Other distributions: log-normal, gamma
In additional experiments, we report results based on density ratio estimation by nonlinear kernel logistic regression.
We use the following two-dimensional polynomial and cubic spline kernels.
\begin{align}
    K_{\mathrm{polynomial}}(\bm{u}, \bm{v}) \coloneqq (\bm{u}^\top\bm{v} + c)^2, \nonumber \\
    K_{\mathrm{spline}}(\bm{u}, \bm{v}) \coloneqq \|\bm{u} - \bm{v} \|^3.
\end{align}

Tables~\ref{tab:dimension_alpha_poly},~\ref{tab:dimension_alpha_spline} and Figures~\ref{fig:poly_gaussian},~\ref{fig:spline_gaussian} show the results of density ratio estimation based on nonlinear kernel logistic regression.
The settings in these experiments are the same as in the main manuscript.
These results show that density ratio estimation along geodesics with large $\alpha$ reduces variances, even when nonlinear kernels are used.

\clearpage

\section{Related Work}
\label{apd:related_work}
\subsection{Applications of Density Ratio}
The density ratio of the two probability distributions has many known applications.
\citet{shimodaira2000improving} shows that in supervised learning under the covariate shift assumption, where the marginal distribution of the input data differs between training and testing, learning by importance weighting using density ratios restores consistency.
This framework is called Importance Weighted Empirical Risk Minimization (IWERM), and many variants have been investigated~\citep{chen2016robust,kimura2022information,kimura2024short,liu2014robust,martin2023double,sugiyama2007covariate,yamada2013relative}.
Density ratios are also known to be effective for outlier detection, and many studies have used estimated density ratios from samples as outlier scores~\citep{chen2015abnormal,hido2011statistical,kanamori2008efficient,li2022robust}.
Another useful application of the density ratio is the two-sample test, which considers whether, given two sample pairs, they are generated from the same distribution~\citep{cheng2004semiparametric,kanamori2011f,keziou2005test,keziou2008empirical,lu2020new}.
Further applications of density ratios include variable selection~\citep{oh2016bayesian}, dimensionality reduction~\citep{suzuki2010sufficient}, causal inference~\citep{matsushita2023estimating}, and estimation of mutual information~\citep{braga2014constructive,suzuki2008approximating}.
Many of them are also discussed theoretically in the use of density ratios.

\subsection{Methods of Density Ratio Estimation}
Because of this wide range of applications, density ratio estimation has become one of the most important tasks.
The simplest idea is to estimate the density ratio based on a separate density estimation, but these two-step approaches are known to be computationally unstable~\citep{sugiyama2012density}.
A more promising idea is a framework for direct density ratio estimation.
One idea is the moment matching approach which tries to match the moments of $p_s(\bm{x})$ and $p_t(\bm{x})$ by considering $\hat{r}(\bm{x})$ as a transformation function~\citep{huang2006correcting}.
Another well-known approach to density ratio estimation is to optimize the divergence between one distribution $p_s(\bm{x})$ and the transformed other $\hat{r}(\bm{x})p_t(\bm{x})$~\citep{sugiyama2012density_bregman,sugiyama2007direct}.
Recently, algorithms for density ratio estimation based on generative models have also been proposed~\citep{choi2021featurized}.

%%%%%%%%%%%%%%%%%%%%%%%%%%%%%%%%%%%%%%%%%%%%%%%%%%%%%%%%%%%%

\end{document}